\documentclass{article}
\usepackage[utf8]{inputenc}
\usepackage{caption}
\usepackage{graphicx,subcaption}

\usepackage{amsthm}
\usepackage{amssymb}
\usepackage{amsmath}
\usepackage{color}
\usepackage{hyperref}
\usepackage{url}

\newif\ifarxiv
\arxivtrue

\newif\ifsubmit
\submittrue

\newif\ifwithappendix
\withappendixtrue

\ifarxiv
\usepackage[final,preprint,nonatbib]{nips_2018}
\else
\usepackage{nips_2018}
\fi

\newtheorem{theorem}{Theorem}

\newtheorem{lemma}{Lemma}
\newtheorem{assumption}{Assumption}

\ifsubmit
\newcommand{\yuandong}[1]{}
\newcommand{\ari}[1]{}
\newcommand{\haonan}[1]{}
\else
\newcommand{\yuandong}[1]{\textcolor{red}{[Yuandong: #1]}}
\newcommand{\ari}[1]{\textcolor{blue}{[Ari: #1]}}
\newcommand{\haonan}[1]{\textcolor{green}{[Haonan: #1]}}
\fi

\def\dd{\mathrm{d}}
\def\vw{\mathbf{w}}

\def\vh{\mathbf{h}}

\def\vv{\mathbf{v}}
\def\vg{\mathbf{g}}

\def\vc{\mathbf{c}}
\def\vf{\mathbf{f}}
\def\vp{\mathbf{p}}
\def\vq{\mathbf{q}}

\def\tt#1{{#1}^\circ}
\def\t#1{#1^\circ}

\def\ee#1{\mathbb{E}\left[#1\right]}
\def\ee2#1#2{\mathbb{E}_{#1}\left[#2\right]}

\title{Luck Matters: Understanding Training Dynamics of Deep ReLU Networks}
\author{Yuandong Tian$\quad\quad$ Tina Jiang$\quad\quad$ Qucheng Gong$\quad\quad$ Ari Morcos \\
Facebook AI Research \\
\texttt{\{yuandong, tinayujiang, qucheng, arimorcos\}@fb.com}}

\def\dd{\mathrm{d}}
\def\norm#1{\|#1\|}

\def\vbeta{\boldsymbol{\beta}}

\begin{document}
\maketitle

\begin{abstract}
    We analyze the dynamics of training deep ReLU networks and their implications on generalization capability. Using a teacher-student setting, we discovered a novel relationship between the gradient received by hidden student nodes and the activations of teacher nodes for deep ReLU networks. With this relationship and the assumption of small overlapping teacher node activations, we prove that (1) student nodes whose weights are initialized to be close to teacher nodes converge to them at a faster rate, and (2) in over-parameterized regimes and 2-layer case, while a small set of lucky nodes do converge to the teacher nodes, the fan-out weights of other nodes converge to zero. This framework provides insight into multiple puzzling phenomena in deep learning like over-parameterization, implicit regularization, lottery tickets, etc. We verify our assumption by showing that the majority of BatchNorm biases of pre-trained VGG11/13/16/19 models are negative. Experiments on (1) random deep teacher networks with Gaussian inputs, (2) teacher network pre-trained on CIFAR-10 and (3) extensive ablation studies validate our multiple theoretical predictions. Code is available at \url{https://github.com/facebookresearch/luckmatters}. 
\end{abstract}

\section{Introduction}

Although neural networks have made strong empirical progress in a diverse set of domains (e.g., computer vision~\cite{alexnet,vgg,resnet}, speech recognition~\cite{hinton2012deep,amodei2016deep}, natural language processing~\cite{word2vec, bert}, and games~\cite{alphago,alphagozero,darkforest, mnih2013playing}), a number of fundamental questions still remain unsolved. How can Stochastic Gradient Descent (SGD) find good solutions to a complicated non-convex optimization problem? Why do neural networks generalize? How can networks trained with SGD fit both random noise and structured data~\cite{rethinking,krueger2017deep,neyshabur2017exploring}, but prioritize structured models, even in the presence of massive noise~\cite{rolnick2017deep}? Why are flat minima related to good generalization? Why does over-parameterization lead to better generalization~\cite{neyshabur2018towards,zhang2019identity,spigler2018jamming,neyshabur2014search,li2018measuring}? Why do lottery tickets exist~\cite{lottery,lottery-scale}?

In this paper, we propose a theoretical framework for multilayered ReLU networks. Based on this framework, we try to explain these puzzling empirical phenomena with a unified view. We adopt a teacher-student setting where the label provided to an \emph{over-parameterized} deep student ReLU network is the output of a fixed teacher ReLU network of the same depth and unknown weights (Fig.~\ref{fig:intro}(a)). Here over-parameterization means that at each layer, the number of nodes in student network is more than the number of nodes in the teacher network. In this perspective, hidden student nodes are randomly initialized with different activation regions (Fig.~\ref{fig:implicit-reg}(a)). During optimization, student nodes compete with each other to explain teacher nodes. From this setting, Theorem~\ref{thm:close-convergence} shows that \emph{lucky} student nodes which have greater overlap with teacher nodes converge to those teacher nodes at a \emph{fast rate}, resulting in \emph{winner-take-all} behavior. Furthermore, Theorem~\ref{thm:over-parameterization} shows that in the 2-layer case, if a subset of student nodes are close to the teachers', they converge to them and the fan-out weights of other irrelevant nodes of the same layer vanishes. 

With this framework, we try to intuitively explain various neural network behaviors as follows: 

\textbf{Fitting both structured and random data}. Under gradient descent dynamics, some student nodes, which happen to overlap substantially with teacher nodes, will move into the teacher node and cover them. This is true for both structured data that corresponds to small teacher networks with few intermediate nodes, or noisy/random data that correspond to large teachers with many intermediate nodes. This explains why the same network can fit both structured and random data (Fig.~\ref{fig:implicit-reg}(a-b)).

\textbf{Over-parameterization}. In over-parameterization, lots of student nodes are initialized randomly at each layer. Any teacher node is more likely to have a substantial overlap with some student nodes, which leads to fast convergence (Fig.~\ref{fig:implicit-reg}(a) and (c), Thm.~\ref{thm:close-convergence}), consistent with \cite{lottery, lottery-scale}. This also explains that training models whose capacity just fit the data (or teacher) yields worse performance~\cite{li2018measuring}. 

\textbf{Flat minima}. Deep networks often converge to ``flat minima'' whose Hessian has a lot of small eigenvalues~\cite{sagun2016eigenvalues,sagun2017empirical,lipton2016stuck,baity2018comparing}. Furthermore, while controversial~\cite{dinh2017sharp}, flat minima seem to be associated with good generalization, while sharp minima often lead to poor generalization~\cite{hochreiter1997flat,keskar2016large,wu2017towards,li2018visualizing}. In our theory, when fitting with structured data, only a few lucky student nodes converge to the teacher, while for other nodes, their fan-out weights shrink towards zero, making them (and their fan-in weights) irrelevant to the final outcome (Thm.~\ref{thm:over-parameterization}), yielding flat minima in which movement along most dimensions  (``unlucky nodes'') results in minimal change in output. On the other hand, sharp minima is related to noisy data (Fig.~\ref{fig:implicit-reg}(d)), in which more student nodes match with the teacher. 


\textbf{Implicit regularization}. On the other hand, the snapping behavior enforces \emph{winner-take-all}: after optimization, a teacher node is fully covered (explained) by a few student nodes, rather than splitting amongst student nodes due to over-parameterization. This explains why the same network, once trained with structured data, can generalize to the test set.

\textbf{Lottery Tickets}. Lottery Tickets~\cite{lottery,lottery-scale,zhou2019deconstructing} is an interesting phenomenon: if we reset ``salient weights'' (trained weights with large magnitude) back to the values before optimization but after initialization, prune other weights (often $> 90\%$ of total weights) and retrain the model, the test performance is the same or better; if we reinitialize salient weights, the test performance is much worse. In our theory, the salient weights are those lucky regions ($E_{j_3}$ and $E_{j_4}$ in Fig.~\ref{fig:lottery-tickets}) that happen to overlap with some teacher nodes after initialization and converge to them in optimization. Therefore, if we reset their weights and prune others away, they can still converge to the same set of teacher nodes, and potentially achieve better performance due to less interference with other irrelevant nodes. However, if we reinitialize them, they are likely to fall into unfavorable regions which cannot cover teacher nodes, and therefore lead to poor performance (Fig.~\ref{fig:lottery-tickets}(c)), just like in the case of under-parameterization. Recently, Supermask~\cite{zhou2019deconstructing} shows that a supermask can be found from winning tickets. If it is applied to initialized weights, the network without training gives much better test performance than chance. This is also consistent with the intuitive picture in Fig.~\ref{fig:lottery-tickets}(b). 

\begin{figure}
    \centering
    \includegraphics[width=0.9\textwidth]{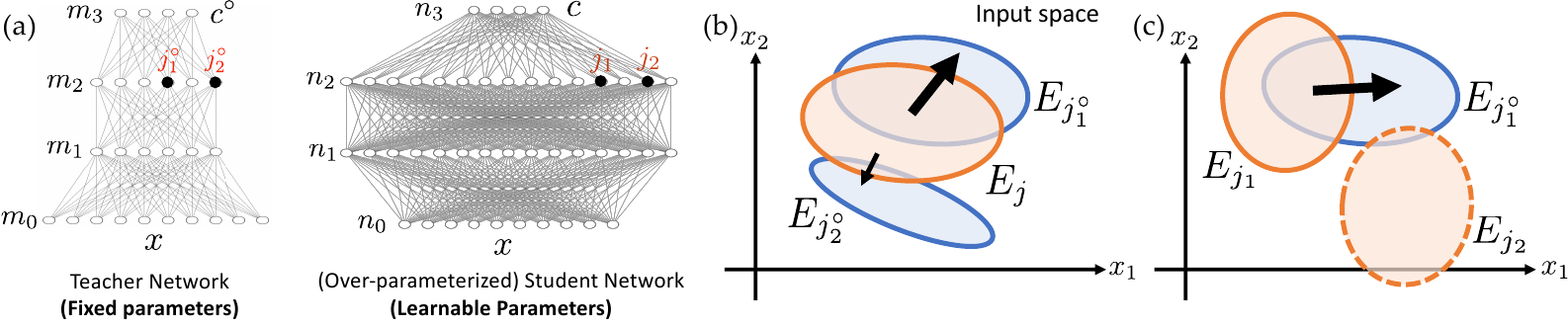}
    \vspace{-0.1in}
    \caption{\small{\textbf{(a)} Teacher-Student Setting. For each node $j$, the activation region is $E_j = \{x: f_j(x) > 0\}$. \textbf{(b)} node $j$ initialized to overlap substantially with a teacher node $\t{j}_1$ converges faster towards $\t{j}_1$ (Thm.~\ref{thm:close-convergence})}. 
    \textbf{(c)} Student nodes initialized to be close to teacher converges to them, while the fan-out weights of other irrelevant student nodes goes to zero. (Thm.~\ref{thm:over-parameterization}).} 
    \label{fig:intro}
\end{figure}

\begin{figure}[t]
    \centering
    \includegraphics[width=\textwidth]{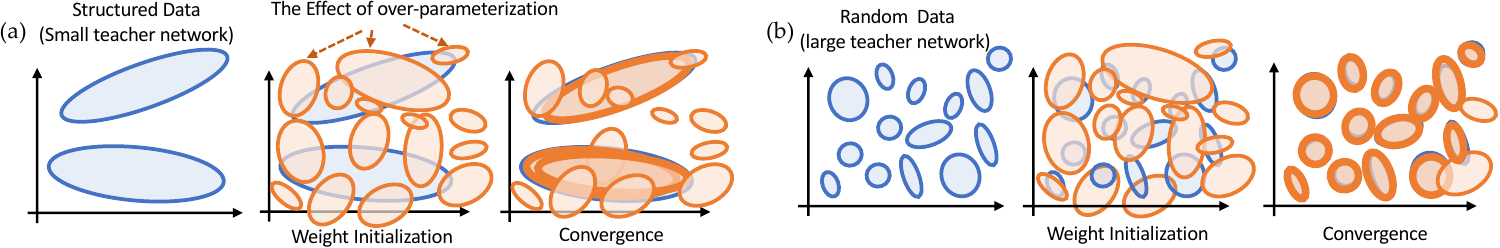}
    \vspace{-0.05in}
    \caption{\small{Explanation of implicit regularization. Blue are activation regions of teacher nodes, while orange are students'.
    \textbf{(a)} When the data labels are structured, the underlying teacher network is small and each layer has few nodes. Over-parameterization (lots of red regions) covers them all. Moreover, those student nodes that heavily overlap with the teacher nodes converge faster (Thm.~\ref{thm:close-convergence}), yield good generalization performance.  
    \textbf{(b)} If a dataset contains random labels, the underlying teacher network that can fit to it has a lot of nodes. Over-parameterization can still handle them and achieves zero training error.}}
    \label{fig:implicit-reg}
    \vspace{-0.05in}
\end{figure}

\begin{figure}[t]
    \centering
    \includegraphics[width=0.8\textwidth]{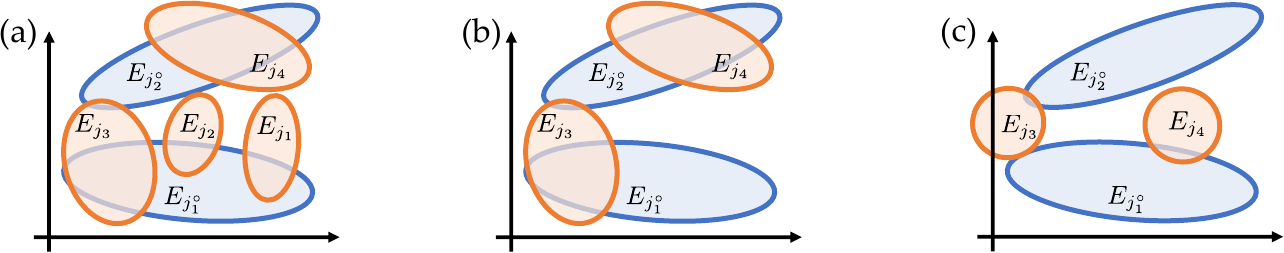}
    \vspace{-0.1in}
    \caption{\small{Explanation of lottery ticket phenomenon. \textbf{(a)} A successful training with over-parameterization (2 filters in the teacher network and 4 filters in the student network). Node $j_3$ and $j_4$ are lucky draws with strong overlap with two teacher node $\t{j}_1$ and $\t{j}_2$, and thus converges with high weight magnitude. \textbf{(b)} Lottery ticket phenomenon: initialize node $j_3$ and $j_4$ with the same initial weight, clamp the weight of $j_1$ and $j_2$ to zero, and retrain the model, the test performance becomes better since $j_3$ and $j_4$ still converge to their teacher node, respectively. \textbf{(c)} If we reinitialize node $j_3$ and $j_4$, it is highly likely that they are not overlapping with teacher node $\t{j}_1$ and $\t{j}_2$ so the performance is not good. }}
    \label{fig:lottery-tickets}
\end{figure}

\vspace{-0.1in}
\section{Mathematical Framework}
\vspace{-0.1in}
\label{sec:framework}

\textbf{Notation}. Consider a student network and its associated teacher network (Fig.~\ref{fig:intro}(a)). Denote the input as $x$. For each node $j$, denote $f_j(x)$ as the activation, $f'_j(x)$ as the ReLU gating (for the top-layer, $f_c'(x)$ and $f_{c^\circ}'(x)$ are always $1$), and $g_j(x)$ as the backpropagated gradient, all as functions of $x$. We use the superscript $\t{}$ to represent a teacher node (e.g., $\t{j}$). Therefore, $g_{\t{j}}$ never appears as teacher nodes are not updated. We use $w_{jk}$ to represent weight between node $j$ and $k$ in the student network. Similarly, $w^*_{\t{j}\t{k}}$ represents the weight between node $\t{j}$ and $\t{k}$ in the teacher network. 

We focus on multi-layered ReLU networks. We use the following equality extensively: $\sigma(x) = \sigma'(x)x$. For ReLU node $j$, we use $E_j \equiv \{x: f_j(x) > 0\}$ as the activation region of node $j$.

\textbf{Teacher network versus Dataset}. The reason why we formulate the problem using teacher network rather than a dataset is the following: (1) It leads to a nice and symmetric formulation for multi-layered ReLU networks (Thm.~\ref{thm:same-layer-relationship}). (2) A teacher network corresponds to an infinite size dataset, which separates the finite sample issues from induction bias in the dataset, which corresponds to the structure of teacher network. 
(3)	If student weights can be shown to converge to teacher ones, generalization bound can naturally follow for the student.  
(4)	The label complexity of data generated from a teacher is automatically reduced, which could lead to better generalization bound. On the other hand, a bound for arbitrary function class can be hard.  

\textbf{Objective}. We assume that both the teacher and the student output probabilities over $C$ classes. We use the output of teacher as the input of the student. At the top layer, each node $c$ in the student corresponds to each node $\t{c}$ in the teacher. Therefore, the objective is:
\begin{equation}
    \min_{\vw} J(\vw) = \frac{1}{2}\ee2{x}{\|f_c(x) - f_{\t{c}}(x)\|^2} \label{eq:loss}
\end{equation}

By the backpropagation rule, we know that for each sample $x$, the (negative) gradient $g_c(x) \equiv \partial J/\partial f_c = f_{\t{c}}(x) - f_c(x)$. The gradient gets backpropagated until the first layer is reached. 

Note that here, the gradient $g_c(x)$ sent to node $c$ is \emph{correlated} with the activation of the corresponding teacher node $f_{\t{c}}(x)$ and other student nodes at the same layer. Intuitively, this means that the gradient ``pushes'' the student node $c$ to align with class $\t{c}$ of the teacher. If so, then the student learns the corresponding class well. A natural question arises: 

\begin{center}
\emph{Are student nodes at intermediate layers correlated with teacher nodes at the same layers?}
\end{center}

One might wonder this is hard since the student's intermediate layer receives no \emph{direct supervision} from the corresponding teacher layer, but relies only on backpropagated gradient. Surprisingly, the following theorem shows that it is possible for every intermediate layer:

\begin{theorem}[Recursive Gradient Rule]
\label{thm:same-layer-relationship}
    If all nodes $j$ at layer $l$ satisfies Eqn.~\ref{eq:compatibility} 
    \begin{equation}
    g_j(x) = f'_j(x)\left[\sum_{\t{j}} \beta^*_{j\t{j}}(x) f_{\t{j}}(x) - \sum_{j'} \beta_{jj'}(x) f_{j'}(x) \right], \label{eq:compatibility}
    \end{equation}
    then all nodes $k$ at layer $l - 1$ also satisfies Eqn.~\ref{eq:compatibility} with $\beta^*_{k\t{k}}(x)$ and $\beta_{kk'}(x)$ defined recursively from top to bottom-layer as follows: 
    \begin{equation}
    \beta^*_{k\t{k}}(x) \equiv \sum_{j\t{j}} w_{jk} f'_j(x)\beta^*_{j\t{j}}(x) f'_{\t{j}}(x) w^*_{\t{j}\t{k}},\quad     \beta_{kk'}(x) \equiv \sum_{jj'} w_{jk} f'_j(x)\beta_{jj'}(x) f'_{j'}(x) w_{j'k'}
    \end{equation}
    And for the base case where node $c$ and $c^\circ$ are at the top-most layer, $\beta^*_{c\t{c}}(x) = \beta_{c\t{c}}(x) = \delta(c-\t{c})$ (i.e., $1$ when $c$ corresponds to $c^\circ$, and $0$ otherwise). 
\end{theorem}

Note that Theorem~\ref{thm:same-layer-relationship} applies to arbitrarily deep ReLU networks and allows different number of nodes for the teacher and student. The role played by ReLU activation is to make the expression of $\beta$ concise, otherwise $\beta_{kk^\circ}$ and $\beta^*_{kk^\circ}$ can take a very complicated (and asymmetric) form.  

In particular, we consider the \emph{over-parameterization} setting: the number of nodes on the student side is much larger (e.g., 5-10x) than the number of nodes on the teacher side. Using Theorem~\ref{thm:same-layer-relationship}, we discover a novel and concise form of gradient update rule:

\label{sec:matrix-formulation}
\begin{assumption}[Separation of Expectations]
\label{assumption:separation}
\begin{eqnarray}
    \ee2{x}{\beta^*_{j\t{j}}(x)f'_j(x)f'_{\t{j}}(x)f_k(x) f_{\t{k}}(x)} &=& \ee2{x}{\beta^*_{j\t{j}}(x)}\ee2{x}{f'_j(x)f'_{\t{j}}(x)}\ee2{x}{f_k(x) f_{\t{k}}(x)} \\
    \ee2{x}{\beta_{jj'}(x)f'_j(x)f'_{j'}(x)f_k(x) f_{k'}(x)} &=& \ee2{x}{\beta_{jj'}(x)}\ee2{x}{f'_j(x)f'_{j'}(x)}\ee2{x}{f_k(x) f_{k'}(x)}
\end{eqnarray}
\end{assumption}

\begin{theorem}
\label{thm:matrix-formulation}
If Assumption~\ref{assumption:separation} holds, the gradient dynamics of deep ReLU networks with objective (Eqn.~\ref{eq:loss}) is: 
\begin{equation}
    \dot W_l = L^*_lW^*_lH^*_{l+1} - L_lW_lH_{l+1} \label{eq:weight-update-matrix}
\end{equation}
\end{theorem}
Here we explain the notations. $W^*_l = \left[\vw^*_1, \ldots, \vw^*_{m_l}\right]$ is $m_l$ teacher weights, $\vbeta_{l+1}^* = \ee2{x}{\beta_{jj^\circ}(x)}$, $d^*_{jj^\circ} = \ee2{x}{f'_j(x)f'_{j^\circ}(x)}$ and $D_l^* = [d^*_{jj^\circ}]$, $H_{l+1}^* = [h_{jj^\circ}] = \vbeta_{l+1}^* \circ D_l$, $l^*_{jj^\circ} = \ee2{x}{f_j(x) f_{\t{j}}(x)}$ and $L^*_l = [l^*_{jj^\circ}]$. We can define similar notations for $W$ (which has $n_l$ columns/filters), $\boldsymbol{\beta}$, $D$, $H$ and $L$ (Fig.~\ref{fig:overlaps}(c)). At the lowest layer $l=0$, $L_0 = L_0^*$, at the highest layer $l = l_{\max} - 1$ where there is no ReLU, we have $\vbeta_{l_{\max}} = \vbeta^*_{l_{\max}} = H_{l_{\max}} = H^*_{l_{\max}} = I$ due to Eqn.~\ref{eq:loss}. According to network structure, $\vbeta_{l+1}$ and $\vbeta^*_{l+1}$ only depends on weights $W_{l+1}, \ldots, W_{l_{\max} - 1}$, while $L_l$ and $L^*_l$ only depend on $W_{0}, \ldots, W_{l-1}$. 

\def\bninput#1{f^{(#1)}}
\def\bnzeromean#1{\hat f^{(#1)}}
\def\bnstandard#1{\tilde f^{(#1)}}
\def\bnoutput#1{\bar f^{(#1)}}

\def\bnvinput{\vf}
\def\bnvzeromean{\hat \vf}
\def\bnvstandard{\tilde \vf}
\def\bnvoutput{\bar \vf}
\def\vzero{\mathbf{0}}
\def\vone{\mathbf{1}}

\def\ch{\mathrm{ch}}
\def\pa{\mathrm{pa}}

\vspace{-0.1in}
\section{Analysis on the Dynamics}
\vspace{-0.1in}
\def\proj#1{P^\perp_{\vw_{#1}}}
\def\projs#1{P^\perp_{\vw^*_{#1}}}

\begin{figure}
    \centering
    \includegraphics[width=\textwidth]{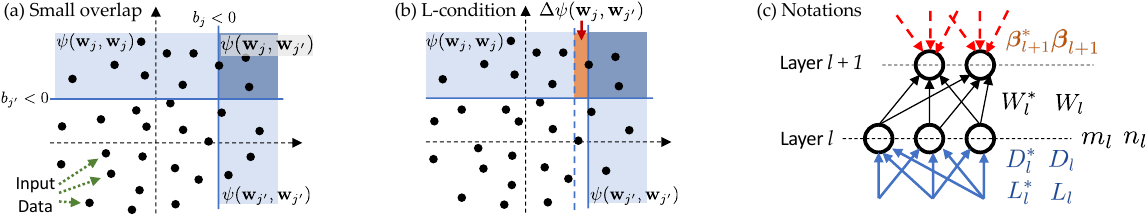}
    \caption{\small{\textbf{(a)} Small overlaps between node activations. Figure drawn in the space spanned by the activations of last layer so all decision boundaries are linear. \textbf{(b)} Lipschitz condition (Assumption~\ref{assumption:L-condition}). \textbf{(c)} Notation in Thm.~\ref{thm:matrix-formulation}.}}
    \vspace{-0.1in}
    \label{fig:overlaps}
\end{figure}

In the following, we will use Eqn.~\ref{eq:weight-update-matrix} to analyze the dynamics of the multi-layer ReLU networks. For convenience, we first define the two functions $\psi_l$ and $\psi_d$ ($\sigma$ is the ReLU function):
\begin{equation}
    \psi_l(\vw, \vw') = \ee2{x}{\sigma(\vw^Tx)\sigma({\vw'}^Tx)}, \quad\psi_d(\vw, \vw') = \ee2{x}{\mathbb{I}(\vw^Tx)\mathbb{I}({\vw'}^Tx)}. 
\end{equation}
We assume these two functions have the following property \yuandong{We should be able to prove them for ReLU later}. 
\begin{assumption}[Lipschitz condition]
\label{assumption:L-condition}
There exists $K_d$ and $K_l$ so that:
\begin{equation}
    \norm{\psi_i(\vw, \vw_1) - \psi_i(\vw, \vw_2)} \le K_i\psi_i(\vw, \vw_1)\norm{\vw_1 - \vw_2}, \quad i \in \{d, l\}
\end{equation}
\end{assumption}

Using this, we know that $d_{jj'} = \psi_d(\vw_j, \vw_{j'})$, $d^*_{jj'} = \psi_d(\vw_j, \vw^*_{j'})$, and so on. For brevity, denote $d^{**}_{jj'} = \psi_d(\vw^*_j, \vw^*_{j'})$ (when notation $j^\circ_1$ is heavy) and so on. We impose the following assumption: 
\begin{assumption}[Small Overlap between teacher nodes]
\label{assumption:weak-interaction}
There exists $\epsilon_l \ll 1$ and $\epsilon_d \ll 1$ so that:
\begin{equation}
d^{**}_{j_1j_2} \le \epsilon_d d^{**}_{j_1j_1} 
\ (\mathrm{or\ } \epsilon_d d^{**}_{j_2j_2}), \quad
l^{**}_{j_1j_2} \le \epsilon_l l^{**}_{j_1j_1} 
\ (\mathrm{or\ } \epsilon_l l^{**}_{j_2j_2}), \quad  \mathrm{for\ } j_1\neq j_2
\end{equation}
\end{assumption}
Intuitively, this means that the probability of the simultaneous activation of two teacher nodes $j_1$ and $j_2$ is small. If we have sufficient training data to cover the input space, then a sufficient condition for Assumption~\ref{assumption:weak-interaction} to happen is that the teacher has negative bias, which means that they \emph{cut corners} in the space spanned by the node activations of the lower layer (Fig.~\ref{fig:overlaps}a). We have empirically verified that the majority of biases in BatchNorm layers (after the data are whitened) are negative in VGG11/16 trained on ImageNet (Sec.~\ref{sec:check-assumption}). 

\subsection{Effects of BatchNorm}
\vspace{-0.1in}
Batch Normalization~\cite{batchnorm} has been extensively used to speed up the training, reduce the tuning efforts and improve the test performance of neural networks. Here we use an interesting property of  BatchNorm: the total ``energy'' of the incoming weights of each node $j$ is conserved over training iterations:

\begin{theorem}[Conserved Quantity in Batch Normalization]
\label{thm:conserved-bn}
For Linear $\rightarrow$ ReLU $\rightarrow$ BN or Linear $\rightarrow$ BN $\rightarrow$ ReLU configuration, $\norm{\vw_j}$ of a filter $j$ before BN remains constant in training.
\ifwithappendix (Fig.~\ref{fig:batchnorm-configuration}).
\fi
\end{theorem}
See Appendix for the proof. The similar lemma is also in~\cite{arora2018theoretical}. This may partially explain why BN has stabilization effect: energy will not leak from one layer to nearby ones. Due to this property, in the following, for convenience we assume $\|\vw_j\|^2 = \|\vw^*_j\|^2 = 1$, and the gradient $\dot \vw_j$ is always orthogonal to the current weight $\vw_j$. Note that on the teacher side we can always push the magnitude component to the upper layer; on the student side, random initialization naturally leads to constant magnitude of weights.  

\subsection{Same number of student nodes as teacher}
\vspace{-0.1in}
We start with a simple case first. Consider that we only analyze layer $l$ without over-parameterization, i.e., $n_l=m_l$. We also assume that $L_l^* = L_l = I$, i.e., the input of that layer $l$ is whitened, and the top-layer signal is uniform, i.e.,  $\boldsymbol{\beta}^*_{l+1} = \boldsymbol{\beta}_{l+1} = \vone\vone^T$ (all $\beta$ entries are 1). Then the following theorem shows that weight recovery could follow (we use $j'$ as $j^\circ$). 

\begin{theorem}
\label{thm:close-convergence}
For dynamics $\dot \vw_j = \proj{j} (W^*\vh_j^* - W\vh_j)$, where $\proj{j} \equiv I - \vw_j\vw^T_j$ is a projection matrix into the orthogonal complement of $\vw_j$. $\vh^*_j$, $\vh_j$ are corresponding $j$-th column in $H^*$ and $H$. Denote $\theta_j = \angle (\vw_j, \vw^*_j)$ and assume $\theta_j \le \theta_0$. If $\gamma = \cos\theta_0 - (m-1)\epsilon_d M_d > 0$, then $\vw_j \rightarrow \vw^*_j$ with the rate $1 - \eta\bar d\gamma$ ($\eta$ is learning rate). Here $\bar d = \left[1 + 2K_d\sin(\theta_0/2)\right]\min_j d^{*0}_{jj}$ and $M_d = (1+K_d)\left[1+2K_d\sin(\theta_0/2)\right]^2 / \cos\frac{\theta_0}{2}$.  
\end{theorem}
See Appendix for the proof. Here we list a few remarks:

\textbf{Faster convergence near $\vw^*_j$}. we can see that due to the fact that $h^*_{jj}$ in general becomes larger when $\vw_j \rightarrow \vw^*_j$ (since $\cos\theta_0$ can be close to $1$), we expect a \emph{super-linear} convergence near $\vw^*_j$. This brings about an interesting \emph{winner-take-all} mechanism: if the initial overlap between a student node $j$ and a particular teacher node is large, then the student node will snap to it (Fig.~\ref{fig:intro}(c)). 

\textbf{Importance of projection operator $\proj{j}$}. Intuitively, the projection is needed to remove any ambiguity related to weight scaling, in which the output remains constant if the top-layer weights are multiplied by a constant $\alpha$, while the low-layer weights are divided by $\alpha$. Previous works~\cite{du2017gradient} also uses similar techniques while we justify it with BN. Without $\proj{j}$, convergence can be harder. 

\textbf{Top-down modulation}. Note that here we assume the top-layer signal is uniform, which means that according to $\beta^*_{kk^\circ}$, there is no preference on which student node $k$ corresponds to which teacher node $k^\circ$. If there is a preference (e.g., $\beta^*_{kk^\circ} = \delta(k - k^\circ)$), then from the proof, the cross-term $h^*_{kk^\circ}$ will be suppressed due to $\beta^*_{kk^\circ}$, making convergence easier. As we will see next, such a \emph{top-down modulation} plays an important role for 2-layer and over-parameterization case. We believe that it also plays a similar role for deep networks. 

\subsection{Over-Parameterization and Top-down Modulation in 2-layer Network}
\vspace{-0.1in}
In the over-parameterization case ($n_l > m_l$, e.g., 5-10x), we arrange the variables into two parts: $W = [W_u, W_r]$, where $W_u$ contains $m_l$ columns (same size as $W^*$), while $W_r$ contains $n_l - m_l$ columns. We use $[u]$ (or $u$-set) to specify nodes $1\le j\le m$, and $[r]$ (or $r$-set) for the remaining part. 

In this case, if we want to show ``the main component'' $W_u$ converges to $W^*$, we will meet with one core question: to where will $W_r$ converge, or whether $W_r$ will even converge at all? We need to consider not only the dynamics of the current layer, but also the dynamics of the upper layer. Using a 1-hidden layer over-parameterized ReLU network as an example, Theorem~\ref{thm:over-parameterization} shows that the upper-layer dynamics $\dot V = L^*V^* - LV$ automatically apply \emph{top-down modulation} to suppress the influence of $W_r$, regardless of their convergence. Here $V = \left[\begin{array}{c} V_u \\ V_r \end{array}\right]$, where $V_u$ are the weight components of $u$-set. See Fig.~\ref{fig:over-parameterization}.

\begin{figure}
    \centering
    \includegraphics[width=\textwidth]{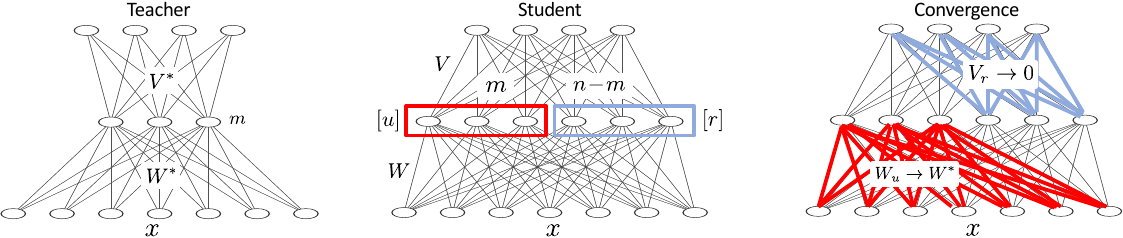}
    \caption{\small{Over-parameterization and top-down modulation. Thm.~\ref{thm:over-parameterization} shows that under certain conditions, the relevant weights $W_u\rightarrow W^*$ and weights connecting to irrelevant student nodes $V_r\rightarrow 0$.}}
    \label{fig:over-parameterization}
    \vspace{-0.2in}
\end{figure}

\begin{theorem}[Over-Parameterization and Top-down Modulation]
\label{thm:over-parameterization}
Consider $\dot W = W^*H^* - WH$ with over-parameterization ($n > m$) and its upper-layer dynamics $\dot V = L^*V^* - LV$. Assume that initial value $W^0_u$ is close to $W^*$: $\theta_j = \angle (\vw_j,  \vw^*_j) \le \theta_0$ for $j \in [u]$. If (1) Assumption~\ref{assumption:weak-interaction} holds for all pairwise combination of columns of $W^*$ \textbf{and} $W^0_r$, and (2) there exists $\gamma = \gamma(\theta_0, m) > 0$ and $\bar\lambda$ 
\ifwithappendix
so that Eqn.~\ref{eq:w-cond} and Eqn.~\ref{eq:v-cond} holds,
\else
that satisfy certain conditions (in Appendix), 
\fi
then $W_u \rightarrow W^*$, $V_u \rightarrow V^*$ and $V_r \rightarrow 0$ with rate $1 - \eta\bar\lambda\gamma$. 
\end{theorem}

See Appendix for the proof (and definition of $\bar\lambda$
\ifwithappendix
in Eqn.~\ref{eq:lambda-bar}\fi). The intuition is: if $W_u$ is close to $W^*$ and $W_r$ are far away from them due to Assumption~\ref{assumption:weak-interaction}, the off-diagonal elements of $L$ and $L^*$ are smaller than diagonal ones. This causes $V_u$ to move towards $V^*$ and $V_r$ to move towards zero. When $V_r$ becomes small, so does $\beta_{jj'}$ for $j\in [r]$ or $j'\in[r]$. This in turn suppresses the effect of $W_r$ and accelerates the convergence of $W_u$. $V_r\rightarrow 0$ exponentially so that $W_r$ stays close to its initial locations, and Assumption~\ref{assumption:weak-interaction} holds for all iterations. A few remarks:

\textbf{Flat minima}. Since $V_r \rightarrow 0$, $W_r$ can be changed arbitrarily without affecting the outputs of the neural network. This could explain why there are many flat directions in trained networks, and why many eigenvalues of the Hessian are close to zero~\cite{sagun2016eigenvalues}. 

\textbf{Understanding of pruning methods}. Theorem~\ref{thm:over-parameterization} naturally relates two different unstructured network pruning approaches: pruning small weights in magnitude~\cite{han2015learning, lottery} and pruning weights suggested by Hessian~\cite{lecun1990optimal, hassibi1993optimal}. It also suggests a principled structured pruning method: instead of pruning a filter by checking its weight norm, pruning accordingly to its top-down modulation. 

\textbf{Accelerated convergence and learning rate schedule}. For simplicity, the theorem uses a uniform (and conservative) $\gamma$ throughout the iterations. In practice, $\gamma$ is initially small (due to noise introduced by $r$-set) but will be large after a few iterations when $V_r$ vanishes. Given the same learning rate, this leads to accelerated convergence. At some point, the learning rate $\eta$ becomes too large, leading to fluctuation. In this case, $\eta$ needs to be reduced. 

\textbf{Many-to-one mapping}. Theorem~\ref{thm:over-parameterization} shows that under strict conditions, there is one-to-one correspondence between teacher and student nodes. In general this is not the case. Two students nodes can be both in the vicinity of a teacher node $\vw^*_j$ and converge towards it, until that node is fully explained. We leave it to the future work for rigid mathematical analysis of many-to-one mappings. 

\textbf{Random initialization}. One nice thing about Theorem~\ref{thm:over-parameterization} is that it only requires the initial $\norm{W_u - W^*}$ to be small. In contrast, there is \emph{no} requirement for small $\norm{V_r}$. Therefore, we could expect that with more over-parameterization and random initialization, in each layer $l$, it is more likely to find the $u$-set (of fixed size $m_l$), or the \emph{lucky weights}, so that $W_u$ is quite close to $W^*$. At the same time, we don't need to worry about $\norm{W_r}$ which grows with more over-parameterization. Moreover, random initialization often gives orthogonal weight vectors, which naturally leads to Assumption~\ref{assumption:weak-interaction}. 

\subsection{Extension to Multi-layer ReLU networks}
\vspace{-0.05in}
Using a similar approach, we could extend this analysis to multi-layer cases. We \emph{conjecture} that similar behaviors happen: for each layer, due to over-parameterization, the weights of some \emph{lucky} student nodes are close to the teacher ones. While these converge to the teacher, the final values of others \emph{irrelevant} weights are initialization-dependent. If the irrelevant nodes connect to lucky nodes at the upper-layer, then similar to Thm.~\ref{thm:over-parameterization}, the corresponding fan-out weights converge to zero. On the other hand, if they connect to nodes that are also irrelevant, then these fan-out weights are not-determined and their final values depends on initialization. However, it doesn't matter since these upper-layer irrelevant nodes eventually meet with zero weights if going up recursively, since the top-most output layer has no over-parameterization. We leave a formal analysis to future work.

\begin{figure}[t]
    \centering
    \includegraphics[width=0.95\textwidth]{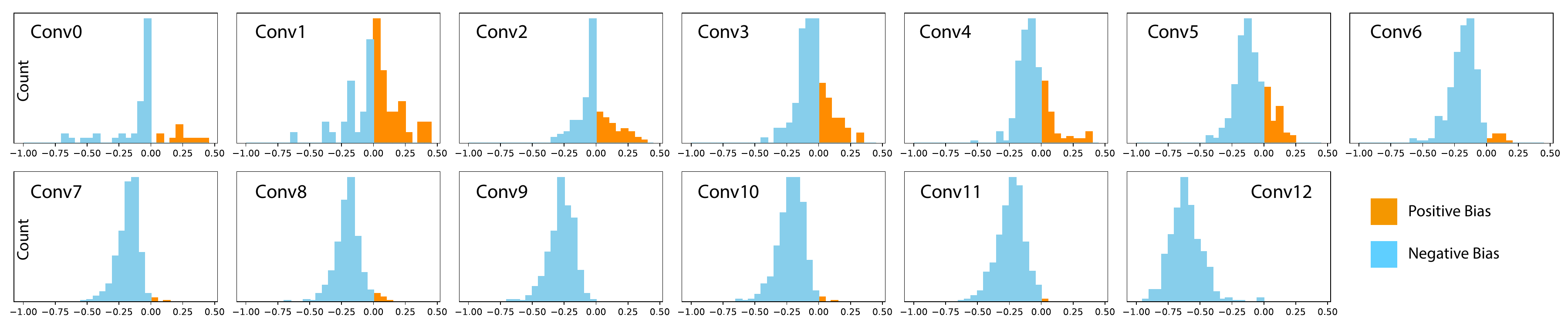}
    \vspace{-0.1in}
    \caption{\small{Distribution of BatchNorm bias in pre-trained VGG16 on ImageNet. Orange/blue are positive/negative biases. \texttt{Conv0} corresponds to the lowest layer (closest to the input). VGG11/13/19 in \ifwithappendix
Fig.~\ref{fig:vgg11-appendix}.
\else
    Appendix.
\fi
 \iffalse \ari{Can we add labels to this plot? e.g. y-axis: Count, x-axis: BN bias value, and then titles for which layer is which?} \yuandong{too small, maybe we only show one for every two layers.}\fi}}
    \label{fig:vgg16}
\end{figure}

\vspace{-0.05in}
\section{Simulations}
\vspace{-0.05in}
\subsection{Checking Assumption~\ref{assumption:weak-interaction}} 
\vspace{-0.05in}
\label{sec:check-assumption}
To make Theorem~\ref{thm:close-convergence} and Theorem~\ref{thm:over-parameterization} work,  we make Assumption~\ref{assumption:weak-interaction} that the activation field of different teacher nodes should be well-separated. To justify this, we analyze the bias of BatchNorm layers after the convolutional layers in pre-trained VGG11/13/16/19. We check the BatchNorm bias $c_1$ as these models use Linear-BatchNorm-ReLU architecture. After BatchNorm first normalizes the input data into zero mean distribution, the BatchNorm bias determines how much data pass the ReLU threshold. If the bias is negative, then a small portion of data pass ReLU gating and Assumption~\ref{assumption:weak-interaction} is likely to hold. From Fig.~\ref{fig:vgg16}, it is quite clear that the majority of BatchNorm bias parameters are negative, in particular for the top layers. 

\subsection{Numerical Experiments of Thm.~\ref{thm:over-parameterization}}
We verify Thm.~\ref{thm:over-parameterization} by checking whether $V_r$ moves close to $0$ under different initialization. We use a network with one hidden layer. The teacher network is 10-20-30, while the student network has more nodes in the hidden layers. Input data are Gaussian noise. We initialize the student networks so that the first $20$ nodes are close to the teacher. Specifically, we first create matrices $W_\epsilon$ and $V_\epsilon$ by first filling with i.i.d Gaussian noise, and then normalizing their columns to $1$. Then the initial value of student $W$ is $W^0_u = \mathrm{column\_normalize}(p_WW^* + W_\epsilon)$, where $p_W$ is a factor controlling how close $W^0_u$ is to $W^*$. For $W_r$ we initialize with noise. Similarly we initialize $V_u$ with a factor $p_V$. The larger $p_W$ and $p_V$, the close the initialization $W^0_u$ and $V^0_u$ to the ground truth values. 

Fig.~\ref{fig:over-param-check} shows the behavior over different iterations. All experiments are repeated 32 times with different random seeds, and (mean$\pm$ std) are reported. We can see that a close initialization leads to faster (and low variance) convergence of $V_r$ to small values. In particular, it is important to have $W^0_u$ close to $W^*$ (large $p_W$), which leads to a clear separation between row norms of $V_u$ and $V_r$, even if they are close to each other at the beginning of training. Having $V^0_u$ close to $V^*$ makes the initial gap larger and also helps convergence. On the other hand, if $p_W$ is small, then even if $p_V$ is large, the gap between row norms of $V_u$ and $V_r$ only shifts but doesn't expand over iterations. 

\begin{figure}
    \centering
    \includegraphics[width=\textwidth]{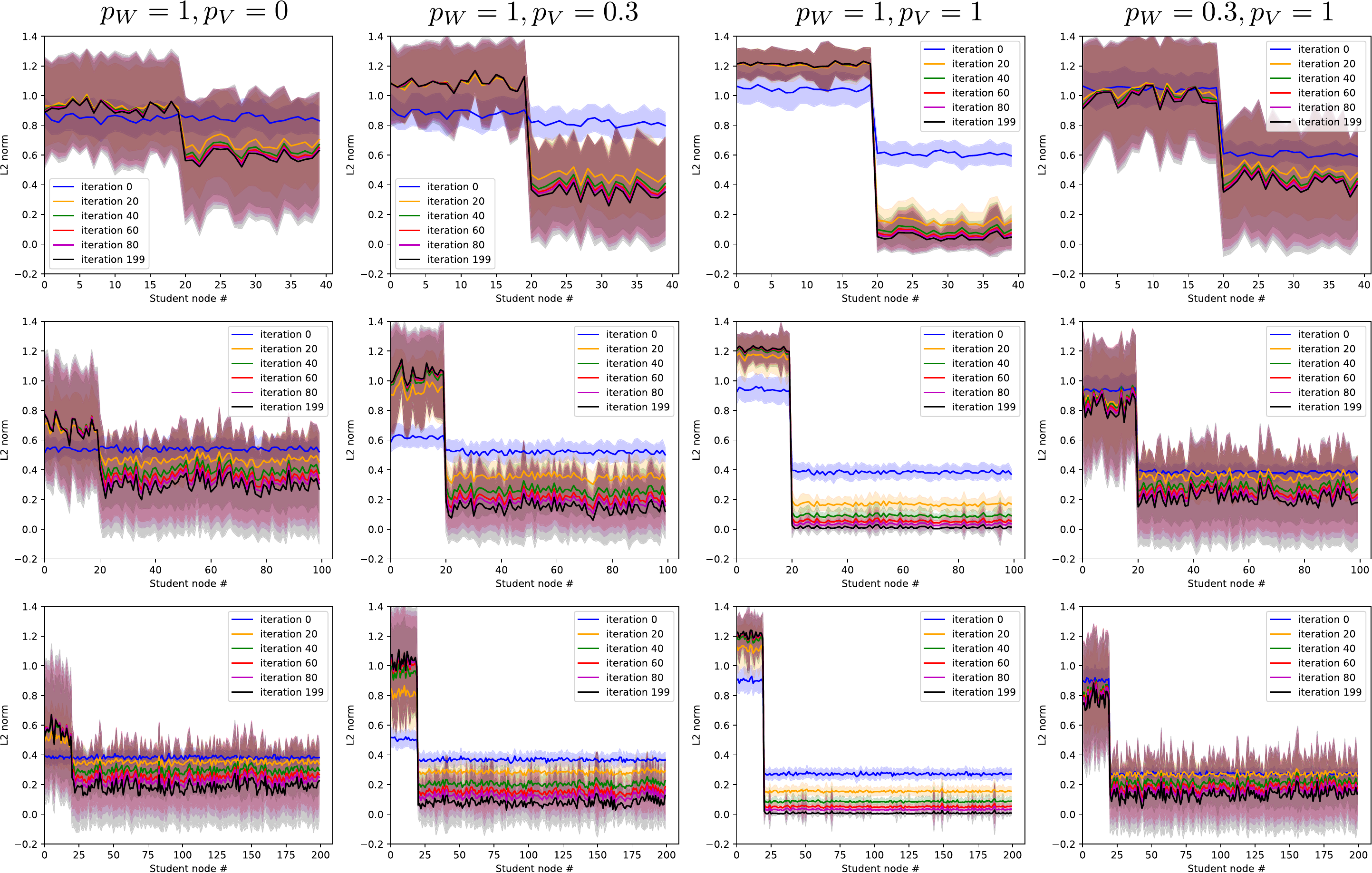}
    \caption{Numerical verification of Thm.~\ref{thm:over-parameterization}. Rows are $2\times$, $5\times$ and $10\times$ over-parameterization. Columns are different initializations of student networks. All figures show $\norm{\vv_j}$, i.e., norm of each row of $V$. }
    \label{fig:over-param-check}
\end{figure}

\section{Experiments}

\def\syth{\texttt{GAUS}}
\def\cifar{\texttt{CIFAR-10}}
\subsection{Experiment Setup}
We evaluate both the fully connected (FC) and ConvNet setting. For FC, we use a ReLU teacher network of size 50-75-100-125. For ConvNet, we use a teacher with channel size 64-64-64-64. The student networks have the same depth but with $10\times$ nodes/channels at each layer, such that they are substnatially over-parameterized. When BatchNorm is added, it is added after ReLU. 

We use random i.i.d Gaussian inputs with mean 0 and std $10$ (abbreviated as \syth) and {\cifar} as our dataset in the experiments. {\syth} generates infinite number of samples while {\cifar} is a finite dataset. For {\syth}, we use a random teacher network as the label provider (with $100$ classes). To make sure the weights of the teacher are weakly overlapped, we sample each entry of $\vw^*_j$ from $[-0.5, -0.25, 0, 0.25, 0.5]$, making sure they are non-zero and mutually different within the same layer, and sample biases from $U[-0.5, 0.5]$. In the FC case, the data dimension is 20 while in the ConvNet case it is $16\times 16$. For {\cifar} we use a pre-trained teacher network with BatchNorm. In the FC case, it has an accuracy of $54.95\%$; for ConvNet, the accuracy is $86.4\%$. We repeat 5 times for all experiments, with different random seed and report min/max values. 

Two metrics are used to check our prediction that some lucky student nodes converge to the teacher:

\textbf{Normalized correlation $\bar\rho$}. We compute normalized correlation (or cosine similarity) $\rho$ between teacher and student activations\footnote{For $\vf_j = [f_j(x_1), \ldots, f_j(x_n)]$ and $\vf_{j^\circ}$, $\rho_{jj^\circ} = \tilde\vf_j^T\tilde\vf_{j^\circ} \in [-1, 1]$, where $\tilde\vf_j = (\vf_j - \mathrm{mean}(\vf_j)) / \mathrm{std}(\vf_j)$.} evaluated on a validation set. At each layer, we average the best correlation over teacher nodes: $\bar\rho = \mathrm{mean}_{j^\circ} \max_j \rho_{jj^\circ}$, where $\rho_{jj^\circ}$ is computed for each teacher and student pairs $(j, j^\circ)$. $\bar\rho \approx 1$ means that most teacher nodes are covered by at least one student. 

\textbf{Mean Rank $\bar r$}. After training, each teacher node $j^\circ$ has the most correlated student node $j$. We check the correlation rank of $j$, normalized to $[0,1]$ ($0$=rank first), back at initialization and at different epochs, and average them over teacher nodes to yield mean rank $\bar r$. Small $\bar r$ means that student nodes that initially correlate well with the teacher keeps the lead toward the end of training. 

\begin{figure}
    \centering
    \includegraphics[width=0.49\textwidth]{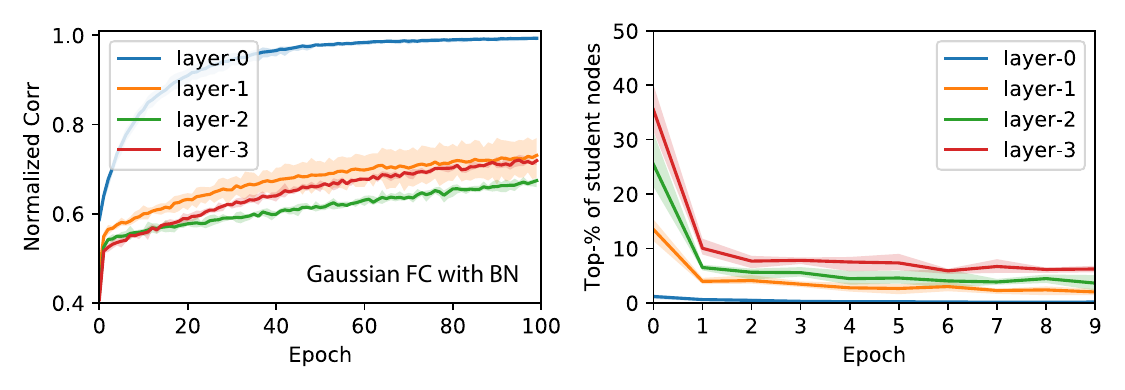}
    \includegraphics[width=0.49\textwidth]{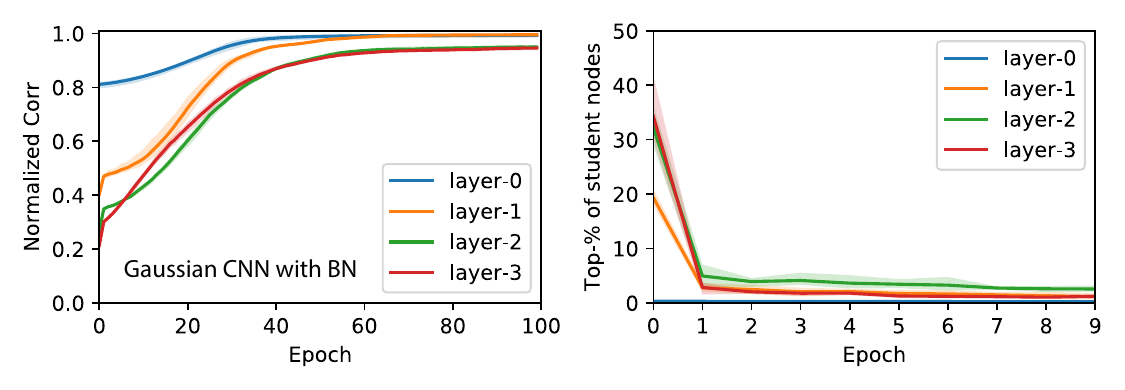}
    \includegraphics[width=0.49\textwidth]{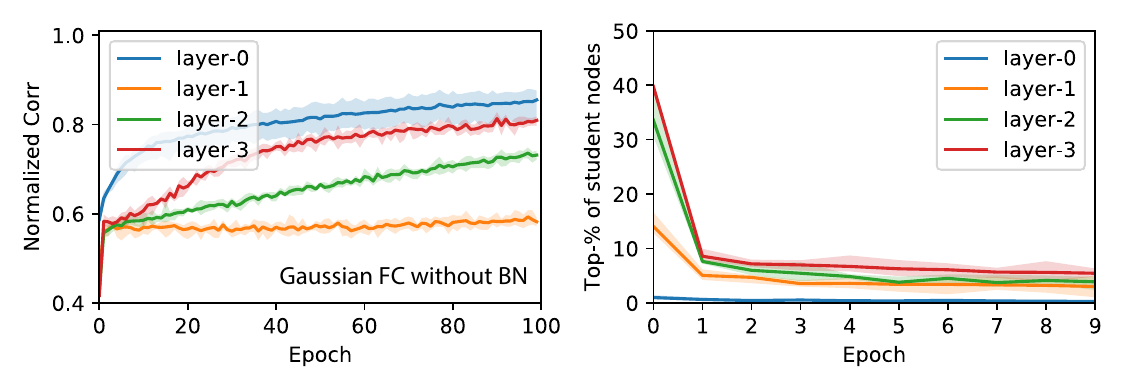}
    \includegraphics[width=0.49\textwidth]{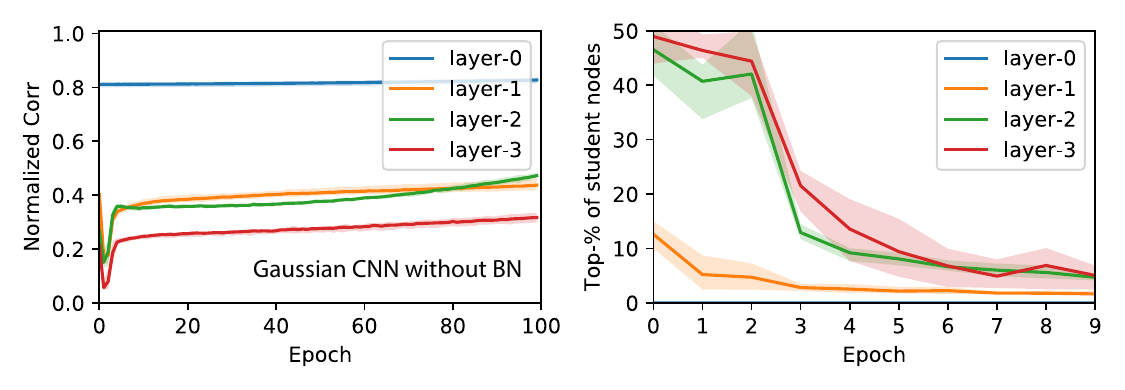}
    \vspace{-0.1in}
    \caption{\small{Correlation $\bar\rho$ and mean rank $\bar r$ over training on \syth{}. $\bar\rho$ steadily grows and $\bar r$ quickly improves over time. \texttt{Layer-0} (the lowest layer that is closest to the input) shows best match with teacher nodes and best mean rank. BatchNorm helps achieve both better correlation and lower $\bar r$, in particular for the CNN case.}}
    \label{fig:exp-gaussian}
\end{figure}

\begin{figure}
    \centering
    \includegraphics[width=0.49\textwidth]{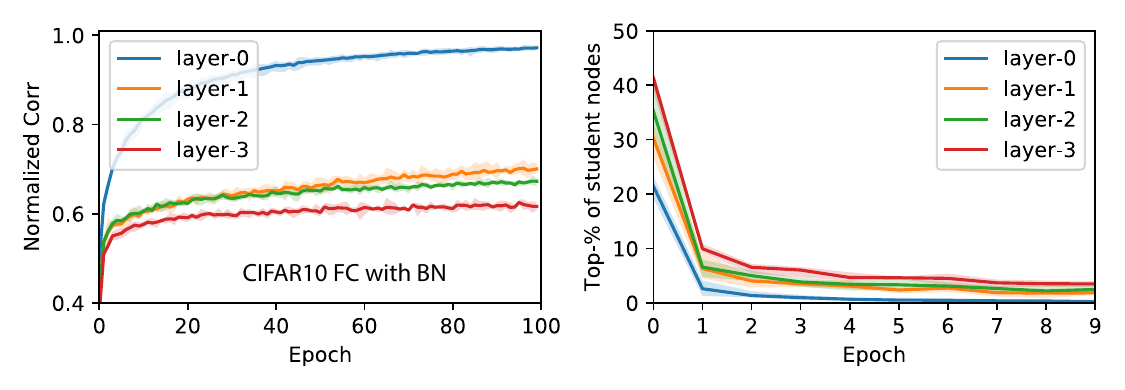}
    \includegraphics[width=0.49\textwidth]{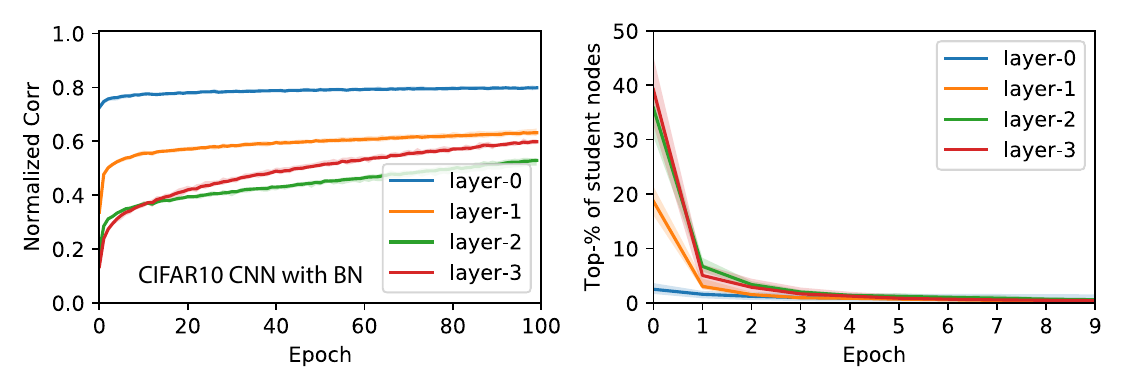}
    \includegraphics[width=0.49\textwidth]{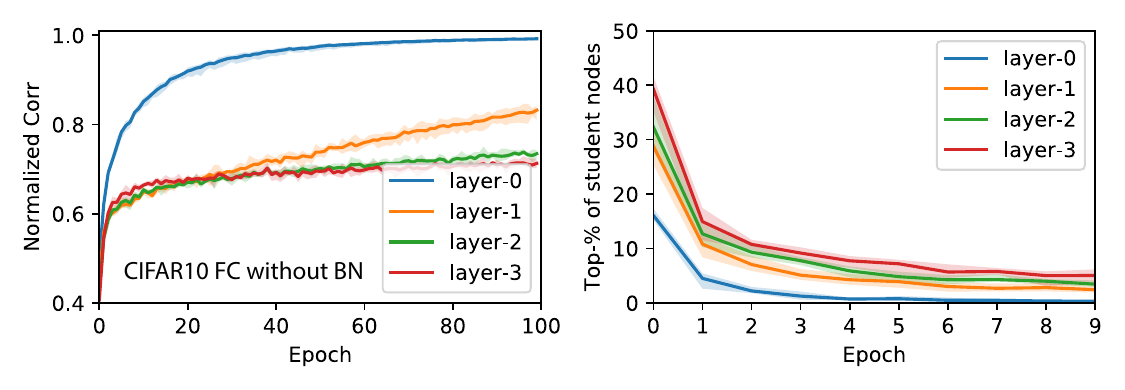}
    \includegraphics[width=0.49\textwidth]{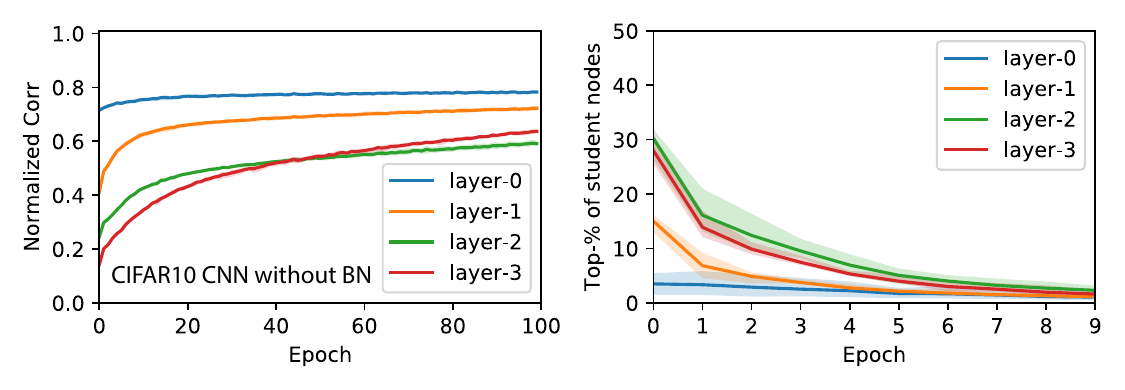}    
    \vspace{-0.1in}
    \caption{\small{Same experiment setting as in Fig.~\ref{fig:exp-gaussian} on \cifar{}. BatchNorm helps achieve lower $\bar r$.}}
    \label{fig:exp-cifar10}
\end{figure}

\begin{figure}
    \centering
    \includegraphics[width=\textwidth]{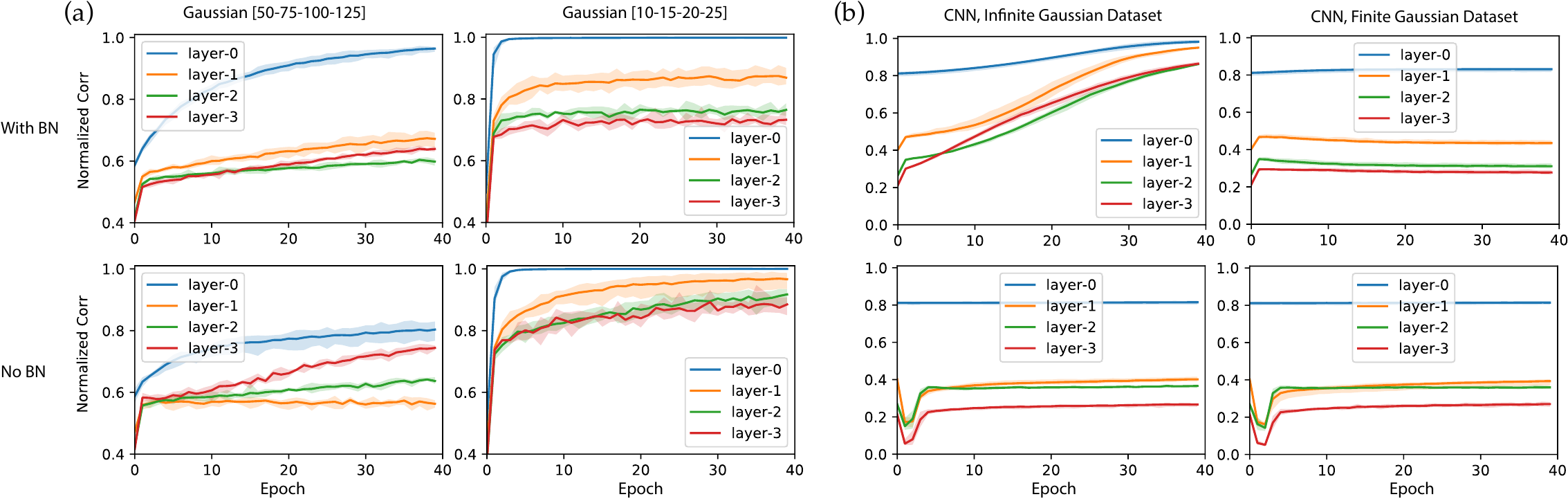}
    \vspace{-0.2in}
    \caption{Ablation studies on \syth{}. \textbf{(a)} $\bar\rho$ converges much faster in small models (10-15-20-25) than in large model (50-75-100-125). BatchNorm hurts in small models. \textbf{(b)} $\bar\rho$ stalls using finite samples.}
    \label{fig:ablation-study}
\end{figure}

\begin{figure}[ht]
    \centering
    \includegraphics[width=\textwidth]{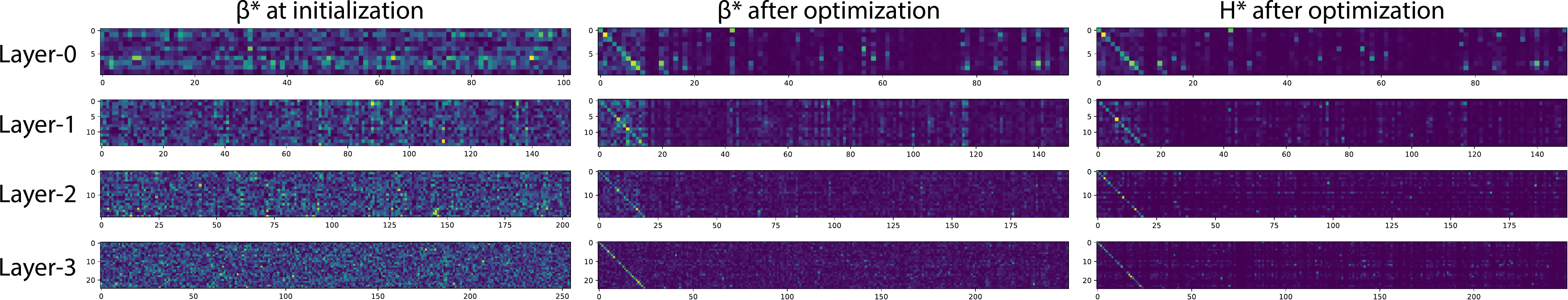}
    \vspace{-0.2in}
    \caption{\small{Visualization of (transpose of) $H^*$ and $\vbeta^*$ matrix before and after optimization (using $\syth$). Student node indices are reordered according to teacher-student node correlations. After optimization, student node who has high correlation with the teacher node also has high $\beta$ entries. Such a behavior is more prominent in $H^*$ matrix that combines $\vbeta^*$ and the activation patterns $D^*$ of student nodes (Sec.~\ref{sec:matrix-formulation}).}}
    \vspace{-0.1in}
    \label{fig:H-reorder}
\end{figure}

\vspace{-0.1in}
\subsection{Results}
\vspace{-0.1in}
Experiments are summarized in Fig.~\ref{fig:exp-gaussian} and Fig.~\ref{fig:exp-cifar10}. $\bar\rho$ indeed grows during training, in particular for low layers that are closer to the input, where $\bar\rho$ moves towards $1$. Furthermore, the final winning student nodes also have a good rank at the early stage of training, in particular after the first epoch, which is consistent with \emph{late-resetting} used in~\cite{lottery-scale}. BatchNorm helps a lot, in particular for the CNN case with {\syth} dataset. For {\cifar}, the final evaluation accuracy (see Appendix) learned by the student is often $\sim 1\%$ higher than the teacher. Using BatchNorm accelerates the growth of accuracy, improves $\bar r$, but seems not to accelerate the growth of $\bar\rho$.

The theory also predicts that the top-down modulation $\vbeta$ helps the convergence. For this, we plot $\beta^*_{jj^\circ}$ at different layers during optimization on {\syth}. For better visualization, we align each student node index $j$ with a teacher node $j^\circ$ according to highest $\rho$. Despite the fact that correlations are computed from the low-layer weights, it matches well with the top-layer modulation (identity matrix structure in Fig.~\ref{fig:H-reorder}). Besides, we also perform ablation studies on \syth{}. 

\textbf{Size of teacher network}. As shown in Fig.~\ref{fig:ablation-study}(a), for small teacher networks (FC 10-15-20-25), the convergence is much faster and training without BatchNorm is faster than training with BatchNorm. For large teacher networks, BatchNorm definitely increases convergence speed and growth of $\bar\rho$. 

\begin{figure}
    \centering
    \includegraphics[width=0.48\textwidth]{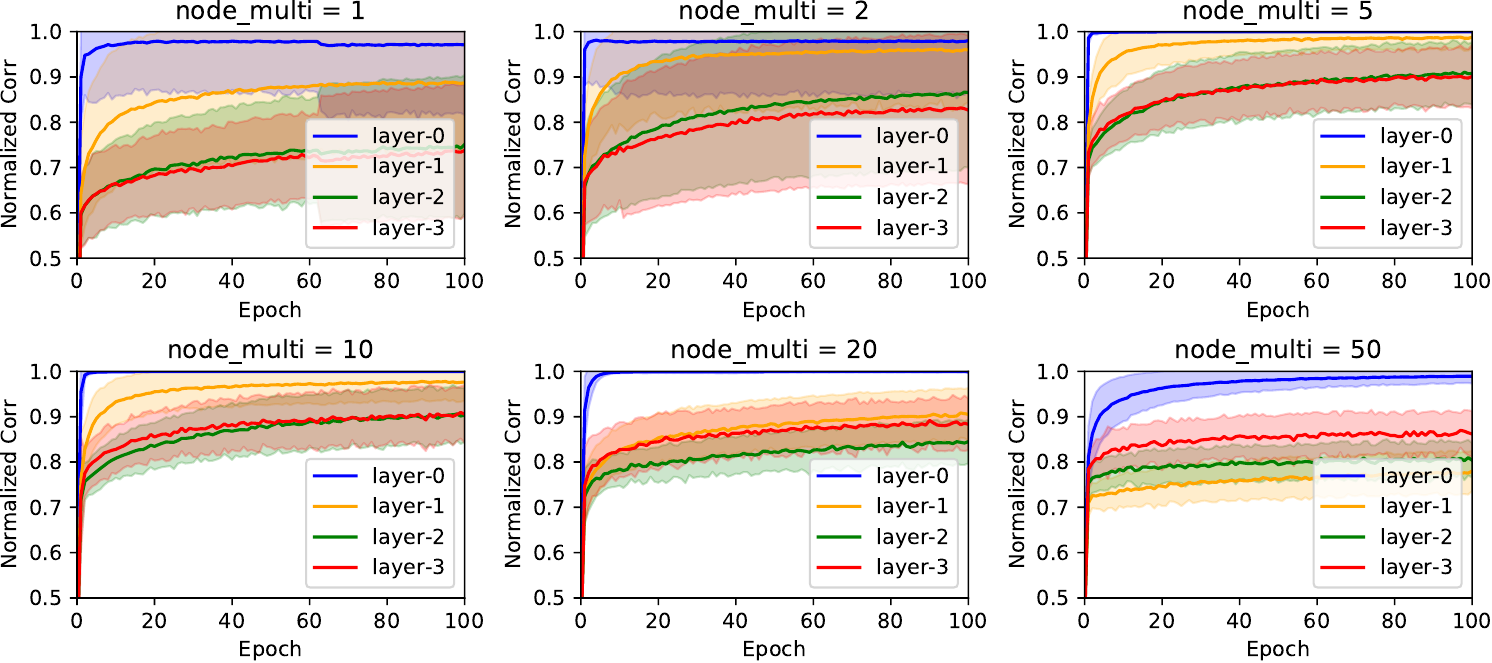}
    \includegraphics[width=0.48\textwidth]{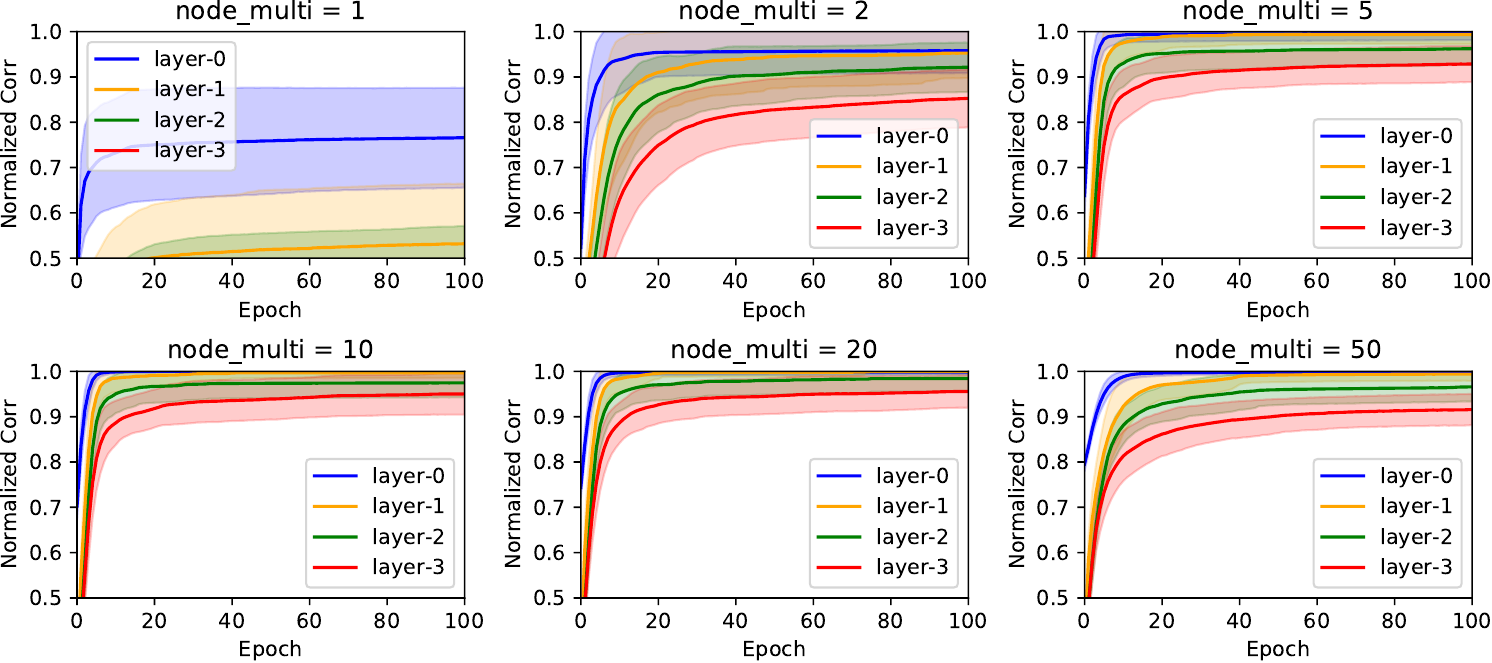}
    \caption{Effect of different degrees of over-parameterization. Left panel: Teacher FC (10-15-20-25) without batchnorm, right panel: teacher CNN (10-15-20-25) with batchnorm. }
    \label{fig:different-over-parameterization}
\end{figure}

\textbf{Degree of over-parameterization}. Fig.~\ref{fig:different-over-parameterization} shows the effects of different degree of over-parameterization ($1\times$, $2\times$, $5\times$, $10\times$, $20\times$ and $50\times$). We initialize 32 different teacher network (10-15-20-25) with different random seed, and plot $\pm$ standard derivation with shaded area. We can clearly see that $\bar\rho$ grows more stably and converges to higher values with over-parameterization. On the other hand, $20\times$ and $50\times$ are slower in convergence due to excessive parameters.  

\textbf{Finite versus Infinite Dataset}. We also repeat the experiments with a pre-generated finite dataset of {\syth} in the CNN case (Fig.~\ref{fig:ablation-study}(b)), and find that the convergence of node similarity stalls after a few iterations. This is because some nodes receive very few data points in their activated regions, which is not a problem for infinite dataset. We suspect that this is probably the reason why {\cifar}, as a finite dataset, does not show similar behavior as {\syth}.

\ifwithappendix
\else
\vspace{-0.2in}
\fi

\section{Conclusion and Future Work}
\vspace{-0.1in}
In this paper we propose a new theoretical framework that uses teacher-student setting to understand the training dynamics of multi-layered ReLU network. With this framework, we are able to conceptually explain many puzzling phenomena in deep networks, such as why over-parameterization helps generalization, why the same network can fit to both random and structured data, why lottery tickets~\cite{lottery, lottery-scale} exist. We backup these intuitive explanations by Theorem~\ref{thm:close-convergence} and Theorem~\ref{thm:over-parameterization}, which collectively show that student nodes that are initialized to be close to the teacher nodes converge to them with a faster rate, and the fan-out weights of irrelevant nodes converge to zero. As the next steps, we aim to extend Theorem~\ref{thm:over-parameterization} to general multi-layer setting (when both $L$ and $H$ are present), relax Assumption~\ref{assumption:weak-interaction} and study more BatchNorm effects than what Theorem~\ref{thm:conserved-bn} suggests. 

\ifarxiv
\section{Acknowledgement}
The first author thanks Simon Du, Jason Lee, Chiyuan Zhang, Rong Ge, Greg Yang, Jonathan Frankle and many others for the informal discussions.
\fi

\bibliographystyle{plain}
\bibliography{main3}

\clearpage

\ifwithappendix
\section{Appendix: Proofs}
\begin{figure}[t]
    \centering
    \includegraphics[width=0.95\textwidth]{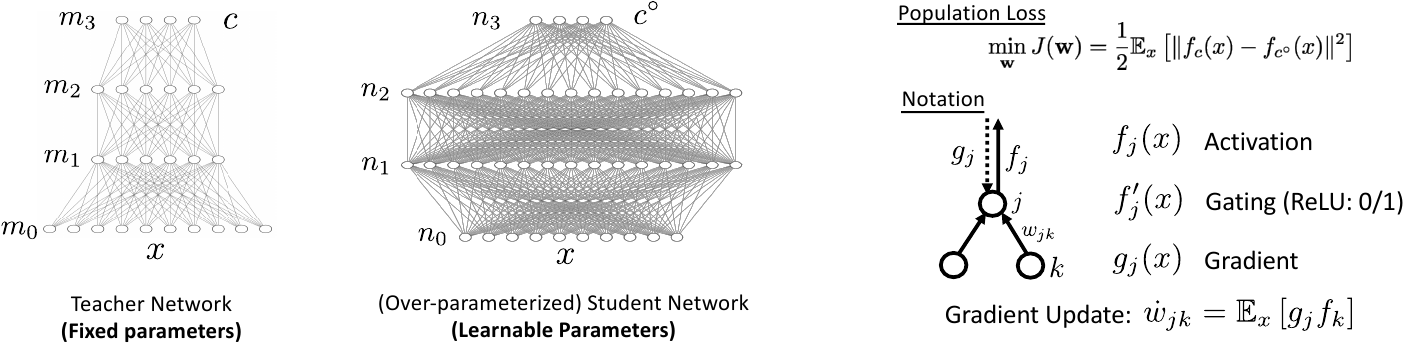}
    \caption{Teacher-Student Setting, loss function and notations.}
    \label{fig:setting}
\end{figure}

\subsection{Theorem~\ref{thm:same-layer-relationship}}
\begin{proof}
The first part of gradient backpropagated to node $j$ is:
\begin{eqnarray}
    g^1_j(x) &=& f'_j(x)\sum_{\t{j}} \beta^*_{j\t{j}}(x)f_{\t{j}}(x) \\ 
    &=& f'_j(x)\sum_{\t{j}} \beta^*_{j\t{j}}(x) f'_{\t{j}}(x)\sum_{\t{k}} w^*_{\t{j}\t{k}}f_{\t{k}}(x) \\
    &=& \sum_{\t{k}}\left[ f_j'(x)\sum_{\t{j}}\beta^*_{j\t{j}}(x) f'_{\t{j}}(x) w^*_{\t{j}\t{k}}\right] f_{\t{k}}(x)
\end{eqnarray}
Therefore, for the gradient to node $k$, we have:
\begin{eqnarray}
g^1_k(x) &=& f'_k(x) \sum_j w_{jk} g^1_j(x) \\ &=& 
f'_k(x) \sum_{\t{k}} \underbrace{\left[ \sum_{j\t{j}} w_{jk} f'_j(x)\beta^*_{j\t{j}}(x) f'_{\t{j}}(x) w^*_{\t{j}\t{k}}\right]}_{\beta^*_{k\t{k}}(x)} f_{\t{k}}(x)
\end{eqnarray}
And similar for $\beta_{kk'}(x)$. Therefore, by mathematical induction, we know that all gradient at nodes in different layer follows the same form.
\end{proof}

\begin{figure}[t]
    \centering
    \includegraphics[width=0.9\textwidth]{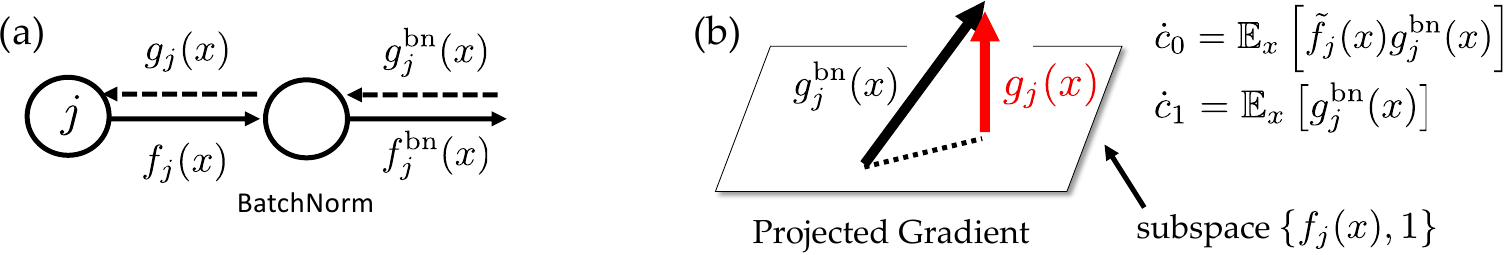}
    \caption{BatchNorm explanation}
    \label{fig:batchnorm-explanation}
\end{figure}

\begin{figure}[t]
    \centering
    \includegraphics[width=0.9\textwidth]{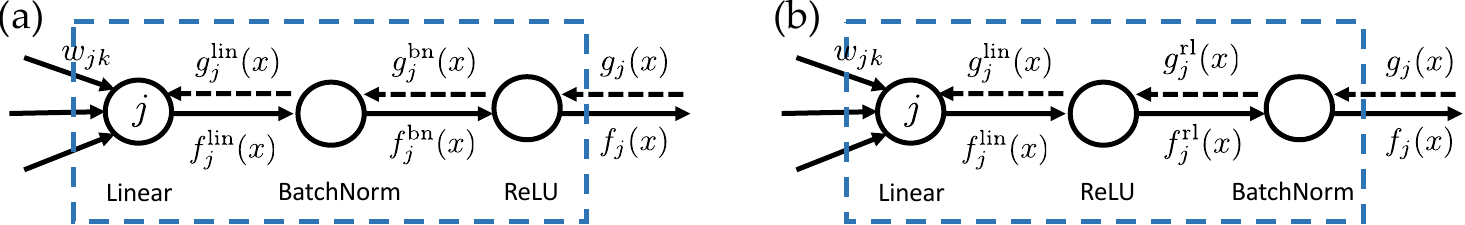}
    \caption{Different BatchNorm Configuration.}
    \label{fig:batchnorm-configuration}
\end{figure}

\subsection{Theorem~\ref{thm:matrix-formulation}} 
\begin{proof}
Using Thm.~\ref{thm:same-layer-relationship}, we can write down weight update for weight $w_{jk}$ that connects node $j$ and node $k$:
\begin{eqnarray}
    \dot w_{jk} &=& \sum_{\t{j},\t{k}} w^*_{\t{j}\t{k}}\underbrace{ \ee2{x}{f'_j(x)\beta^*_{j\t{j}}(x)f'_{\t{j}}(x)f_k(x) f_{\t{k}}(x)}}_{\beta^*_{j\t{j}k\t{k}}}\nonumber \\
    &-& 
\sum_{j',k'} w_{j'k'}\underbrace{ \ee2{x}{f'_j(x) \beta_{jj'}(x)f'_{j'}(x)f_k(x) f_{k'}(x)}}_{\beta_{jj'kk'}} \label{eqn:weight-update}
\end{eqnarray}

Note that $\t{j}$, $\t{k}$, $j'$ and $k'$ run over all parents and children nodes on the teacher side. This formulation works for over-parameterization (e.g., $j^\circ$ and $j'$ can run over different nodes). Applying Assumption~\ref{assumption:separation} and rearrange terms in matrix form yields Eqn.~\ref{eq:weight-update-matrix}.  
\end{proof}

\subsection{Theorem~\ref{thm:conserved-bn}}
\begin{proof}
Given a batch with size $N$, denote pre-batchnorm activations as $\vf = [f_j(x_1), \ldots, f_j(x_N)]^T$ and gradients as $\vg = [g_j(x_1), \ldots, g_j(x_N)]^T$ (See Fig.~\ref{fig:batchnorm-explanation}(a)). $\tilde \vf = (\vf - \mu) / \sigma$ is its whitened version, and $c_0 \tilde\vf + c_1$ is the final output of BN. Here $\mu = \frac{1}{N}\sum_i f_j(x_i)$ and $\sigma^2 = \frac{1}{N}\sum_i (f_j(x_i) - \mu)^2$ and $c_1$, $c_0$ are learnable parameters. With vector notation, the gradient update in BN has a compact form with clear geometric meaning: 

\begin{lemma}[Backpropagation of Batch Norm~\cite{yuandong-relu}]
\label{thm:bn}
For a top-down gradient $\vg$, BN layer gives the following gradient update ($P^\perp_{\bnvinput, \vone}$ is the orthogonal complementary projection of subspace $\{\bnvinput, \vone\}$):
\begin{equation}
    \vg_\vf = J^{BN}(\bnvinput)\vg = \frac{c_0}{\sigma}P^\perp_{\bnvinput, \vone}\vg, \quad \vg_\vc = S(\bnvinput{})^T \vg
    \label{eq:batch-norm-projection}
\end{equation}
\end{lemma}

Intuitively, the back-propagated gradient $J^{BN}(\bnvinput)\vg$ is zero-mean and perpendicular to the input activation $\bnvinput$ of BN layer, as illustrated in Fig.~\ref{fig:batchnorm-explanation}. Unlike~\cite{kohler2018towards, yang2019mean} that analyzes BN in an approximate manner, in Thm.~\ref{thm:bn} we do not impose any assumptions. 

Given Lemma~\ref{thm:bn}, we can prove Thm.~\ref{thm:conserved-bn}. For Fig.~\ref{fig:batchnorm-configuration}(a),
using the property that $\ee2{x}{g^{\mathrm{lin}}_jf^{\mathrm{lin}}_j} = 0$ (the expectation is taken over batch) and the weight update rule $\dot w_{jk} = \ee2{x}{g^{\mathrm{lin}}_jf_k}$ (over the same batch), we have: 
\begin{equation}
    \frac{1}{2}\frac{\dd \norm{\vw_j}^2}{\dd t} = \sum_{k\in\ch(j)} w_{jk}\dot w_{jk} = \ee2{x}{\sum_{k\in \ch(j)} w_{jk} f_k(x)g^{\mathrm{lin}}_j(x)} = \ee2{x}{f^{\mathrm{lin}}_j(x)g^{\mathrm{lin}}_j(x)} = 0 
\end{equation}
For Fig.~\ref{fig:batchnorm-configuration}(b), note that $\ee2{x}{g^{\mathrm{lin}}_jf^{\mathrm{lin}}_j} = \ee2{x}{g^{\mathrm{rl}}_j f^{\mathrm{rl'}}_j f^{\mathrm{lin}}_j} = \ee2{x}{g^{\mathrm{rl}}_j f^{\mathrm{rl}}_j} = 0$ and conclusion follows.
\end{proof}

\subsection{Lemmas}
For simplicity, in the following, we use $\delta\vw_j = \vw_j - \vw^*_j$. 

\begin{lemma}[Bottom Bounds]
\label{lemma:bottom-bounds}
Assume all $\norm{\vw_j} = \norm{\vw_{j'}} = 1$. Denote 
\begin{equation}
    \vp^*_{jj'} \equiv \vw^*_{j'} d^*_{jj'}, \quad \vp_{jj'} \equiv \vw_{j'}d_{jj'}, \quad \Delta \vp_{jj'} \equiv \vp^*_{jj'} - \vp_{jj'}
\end{equation}
If Assumption~\ref{assumption:L-condition} holds, we have: 
\begin{equation}
    \norm{\Delta\vp_{jj'}} \le (1 + K_d)d^*_{jj'}\norm{\delta\vw_{j'}}
\end{equation}
If Assumption~\ref{assumption:weak-interaction} also holds, then:
\begin{equation}
    d^*_{jj'} \le \epsilon_d(1 +K_d\norm{\delta\vw_{j'}})(1 +K_d\norm{\delta\vw_j})d^*_{jj}
\end{equation}
\end{lemma}
\begin{proof}
We have for $j \neq j'$:
\begin{eqnarray}
    \norm{\Delta \vp_{jj'}} &=& 
    \norm{\vw^*_{j'} d^*_{jj'} - \vw_{j'}d_{jj'}} \\ 
    &=& \norm{\vw_{j'}(d^*_{jj'} - d_{jj'}) + (\vw^*_{j'} - \vw_{j'})d^*_{jj'}} \\
    &\le& \norm{\vw_{j'}}\norm{d^*_{jj'} - d_{jj'}} + \norm{\vw^*_{j'} - \vw_{j'}}d^*_{jj'} \\
    &\le& d^*_{jj'}K_d \norm{\delta\vw_{j'}} + d^*_{jj'} \norm{\delta\vw_{j'}} \\
    &\le& (1 + K_d)d^*_{jj'} \norm{\delta\vw_{j'}}
\end{eqnarray}
If Assumption~\ref{assumption:weak-interaction} also holds, we have:
\begin{eqnarray}
    d^*_{jj'} &\le& d^{**}_{jj'}(1 + K_d\norm{\delta\vw_{j'}}) \\
    &\le& \epsilon_d d^{**}_{jj}(1 + K_d\norm{\delta\vw_{j'}}\\
    &\le& \epsilon_d d^{*}_{jj}(1 + K_d\norm{\delta\vw_{j}})(1 + K_d\norm{\delta\vw_{j'}})
\end{eqnarray}
\end{proof}

\begin{lemma}[Top Bounds]
\label{lemma:top-bounds}
Denote 
\begin{equation}
    \vq^*_{jj'} \equiv \vv^*_{j'} l^*_{jj'}, \quad \vq_{jj'} \equiv \vv_{j'}l_{jj'}, \quad \Delta \vq_{jj'} \equiv \vq^*_{jj'} - \vq_{jj'}
\end{equation}
If Assumption~\ref{assumption:L-condition} holds, we have: 
\begin{equation}
    \norm{\Delta\vq_{jj'}} \le (1 + K_l)l^*_{jj'}\norm{\delta\vw_{j'}}
\end{equation}
If Assumption~\ref{assumption:weak-interaction} also holds, then:
\begin{equation}
    l^*_{jj'} \le \epsilon_l(1 +K_l\norm{\delta\vw_{j'}})(1 +K_l\norm{\delta\vw_j})l^*_{jj}
\end{equation}
\end{lemma}
\begin{proof}
The proof is similar to Lemma~\ref{lemma:bottom-bounds}.
\end{proof}

\begin{figure}
    \centering
    \includegraphics{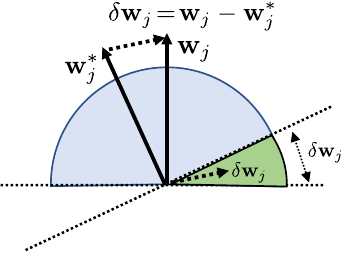}
    \caption{Explanation of Lemma.~\ref{lemma:quadratic-l}.}
    \label{fig:quadratic_l}
\end{figure}

\begin{lemma}[Quadratic fall-off for diagonal elements of $L$]
\label{lemma:quadratic-l}
For node $j$, we have:
\begin{equation}
    \norm{l^*_{jj} - l_{jj}} \le C_0 l^*_{jj}\norm{\delta\vw_j}^2
\end{equation}
\end{lemma}
\begin{proof}
The intuition here is that both the volume of the affected area  and the weight difference are proportional to $\norm{\delta\vw_j}$.  $\norm{l^*_{jj} - l_{jj}}$ is their product and thus proportional to $\norm{\delta\vw_j}^2$. See Fig.~\ref{fig:quadratic_l}.
\yuandong{Need a formal proof.}
\end{proof}

\def\tvp{\tilde\vp}

\subsection{Theorem~\ref{thm:close-convergence}}
\begin{proof}
First of all, note that $\norm{\delta\vw_j} = 2\sin{\frac{\theta_j}{2}} \le 2\sin{\frac{\theta_0}{2}}$. So given $\theta_0$, we also have a bound for $\norm{\delta\vw_j}$. 

When $\boldsymbol{\beta} = \boldsymbol{\beta^*} = \vone\vone^T$, the matrix form can be written as the following: 
\begin{equation}
\dot\vw_j = \proj{j} \vw^*_j h^*_{jj} + \sum_{j'\neq j} \proj{j} \left(\vw^*_{j'} h^*_{jj'} - \vw_{j'}h_{jj'}\right) = \proj{j}\vp^*_{jj} + \sum_{j'\neq j} \proj{j} \Delta\vp_{jj'}
\end{equation}
by using $\proj{j}\vw_j \equiv 0$ (and thus $h_{jj}$ doesn't matter). Since $\norm{\vw_j}$ is conserved, it suffices to check whether the projected weight vector $\projs{j}\vw_j$ of $\vw_j$ onto the complementary space of the ground truth node $\vw^*_j$, goes to zero:

\begin{equation}
    \projs{j}\dot\vw_j = \projs{j}\proj{j} \vp^*_{jj} + \sum_{j'\neq j}\projs{j}\proj{j} \Delta\vp_{jj'}
\end{equation}
Denote $\theta_j = \angle (\vw_j, \vw^*_j)$ and a simple calculation gives that $\sin\theta_j = \norm{\projs{j}\vw_j}$. First we have:
\begin{equation}
    \projs{j}\proj{j} \vw^*_j = \projs{j}(I - \vw_j\vw_j^T)\vw^*_j =  -\projs{j}\vw_j\vw_j^T\vw^
    *_j = -\cos\theta_j \projs{j}\vw_j \label{eq:cos-convergence}
\end{equation}

From Lemma~\ref{lemma:bottom-bounds}, we knows that 
\begin{equation}
    \norm{\Delta\vp_{jj'}} \le (1+K_d)d^*_{jj'}\norm{\delta\vw_{j'}} \le \epsilon_d (1+K_d)\left[1+2K_d\sin(\theta_0/2)\right]^2 d^*_{jj}\norm{\delta\vw_{j'}} 
\end{equation}

Note that here we have $\norm{\delta\vw_{j'}} = 2\sin\frac{\theta_j}{2} = \sin\theta_j / \cos\frac{\theta_j}{2} \le \sin\theta_j / \cos\frac{\theta_0}{2}$. We discuss finite step with very small learning rate $\eta > 0$:
\begin{eqnarray}
    \sin\theta^{t+1}_j &=& \norm{\projs{j}\vw^{t+1}_j} =   \norm{\projs{j}\vw^{t}_j + \eta \projs{j}\dot\vw^t_j} \\ 
    &\le& (1 - \eta d^*_{jj} \cos \theta^t_j)\sin\theta^t_j + \eta \epsilon_dM_d\sum_{j'\neq j} d^*_{jj}\sin\theta^t_{j'} 
 \label{eq:sin-theta-convergence}
\end{eqnarray}
since $\|\projs{j}\| = \|\proj{j}\| = 1$. Here 
\begin{equation}
M_d = (1+K_d)\left[1+2K_d\sin(\theta_0/2)\right]^2 / \cos\frac{\theta_0}{2} \label{eq:md}
\end{equation}
is an iteration independent constant. 

We set $\gamma = \cos\theta_0 - (m-1)\epsilon_dM_d$. If $\gamma  > 0$, denote a constant $\bar d = \left[1 + 2K_d\sin(\theta_0/2)\right]\min_j d^{*0}_{jj}$ and from Lemma~\ref{assumption:L-condition} we know $d^*_{jj} \ge \bar d$ for all $j$. Then given the inductive hypothesis that $\sin\theta^t_j \le (1 - \eta\bar d\gamma)^{t-1} \sin\theta_0$, we have:
\begin{equation}
    \sin\theta^{t+1}_j \le (1 - \eta \bar d\gamma)^{t} \sin\theta_0
\end{equation}
Therefore, $\sin\theta^t_j \rightarrow 0$, which means that $\vw_j \rightarrow \vw^*_j$.
\end{proof}

A few remarks:

\textbf{The projection operator $\proj{j}$.} Note that $\proj{j}$ is important. Intuitively, without the projection, if the same proof logic  worked, one could have concluded that $\vw$ converges to any $\alpha \vw^*$, where $\alpha$ is a constant scaling factor, which is obviously wrong. 

Indeed, without $\proj{j}$, there would be a term $\vw_j^*h_{jj}^* - \vw_j h_{jj}$ on RHS. This term breaks into $\vw_j(h_{jj}^* - h_{jj}) + (\vw_j^* - \vw_j) h^*_{jj}$. Although there could exist $C$ so that $\norm{h_{jj}^* - h_{jj}} \le C\norm{\delta\vw_j}$, unlike Lemma~\ref{lemma:quadratic-l}, $C$ may not be small, and convergence is not guaranteed.  

\subsection{Theorem~\ref{thm:over-parameterization}}
\begin{proof}
First, only for $j\in [u]$, we have their ground truth value $\vw^*_j$. For $j\in[r]$, we assign $\vw^*_j = \vw^0_j$, i.e., their initial values. As we will see, this will make things easier.

From the assumption, we know that $\sin\theta_j \le \sin\theta_0$ for $j\in[u]$. In addition, denote that $\norm{\delta\vv^0_j} \le B_{\delta v}$ for $j\in [u]$. Denote $B_v$ as the bound for all $\norm{\vv^*_j}$. 

Now suppose we can find a $\gamma > 0$ if the following set of equations are satisfied:
\begin{eqnarray}
    \gamma &\ge& (B_v - B_{\delta v})\cos\theta_0 - \epsilon_d (B_v + B_{\delta v}) \max(B_{d,u}, B_{d,r}) > 0 \label{eq:w-cond} \\
    \gamma &\ge& 1 - \epsilon_l \max(B_{l,u}, B_{l,r}) - \kappa > 0 \label{eq:v-cond}
\end{eqnarray}
Here 
\begin{eqnarray}
    \bar d &=& (1 - K_d C_{d,j})\min_j d^{*0}_{jj} > 0 \\
    \bar l &=& (1 - K_l C_{l,j})\min_j l^{*0}_{jj} > 0 \\
    \bar\lambda &=& \min(\bar d, \bar l) \label{eq:lambda-bar} \\
    \kappa &=& 2C_0\sin(\theta_0/2)(1+B_{\delta v}) \\ 
    C_{d,u} &=& 2K_d\sin(\theta_0/2) \\
    C_{d,r} &=& \epsilon_d K_d \frac{B_{d,r}(B_v + B_{\delta v})B_v}{\bar\lambda\gamma(2 - \eta\bar\lambda\gamma)} \\
    M_d^{uu} &=& (1+K_d)(1 + C_{d,u})^2 / \cos\frac{\theta_0}{2}  \\
    M_d^{ur} &=& (1+K_d)(1 + C_{d,u})(1 + C_{d,r})  \\
    M_d^{ru} &=& (1+K_d)(1 + C_{d,u})(1 + C_{d,r}) / \cos\frac{\theta_0}{2} \\
    M_d^{rr} &=& (1+K_d)(1 + C_{d,r})^2 \\
    B_{d,u} &=& (m-1)M_d^{uu} + (n-m)M_d^{ur} \\
    B_{d,r} &=& (m-1)M_d^{ru} + (n-m)M_d^{rr}
\end{eqnarray}
and similarly we can define $C_l$ and $M_l$ etc. If we can find such a $\gamma > 0$ then the dynamics converges. Here all $C$ are close to 0 and $M$ are close to 1. 

Note that if $\epsilon_d$ and $\epsilon_l$ are small, it is obvious to see there exists a feasible $\gamma > 0$ (e.g., $\gamma = 1$). 

To prove it, we maintain the following induction hypothesis for iteration $t$ \yuandong{Explain what is $M_{d,jj}$}:
\begin{equation}    
    d^{*t}_{jj'} \le \epsilon_d M_{d,jj'} d^{*t}_{jj}, \ \ 
    l^{*t}_{jj'} \le \epsilon_l M_{l,jj'} l^{*t}_{jj},
    \quad j'\neq j \label{eq:weight-separation} \tag{W-Separation}
\end{equation}
\begin{equation}
    \sin\theta^t_j \le (1 - \eta \bar d\gamma)^{t-1} \sin\theta_0,\quad j\in[u] \label{eq:wu-contract} \tag{$W_u$-Contraction}
\end{equation}
\begin{equation}
    \norm{\delta\vv^t_j} \le (1 - \eta\bar l\gamma)^{t-1}B_{\delta v},\ \ j\in[u], \quad\quad
    \norm{\vv^t_j} \le (1 - \eta\bar l \gamma)^{t-1}B_v,\ \  j\in[r] \label{eq:v-contract} \tag{$V$-Contraction}
\end{equation}

Besides, the following condition is involved (but it is not part of induction hypothesis):
\begin{equation}
    \norm{\vw^t_j - \vw^0_j} \le C_{d,r}, \quad j\in[r]\tag{$W_r$-Bound}\label{eq:wr-bound}
\end{equation}
\begin{equation}
    d^{*t}_{jj} \ge d^{*0}_{jj}(1 - K_d C_{d,j}) \ge \bar d > 0, \quad l^{*t}_{jj} \ge l^{*0}_{jj}(1 - K_l C_{l,j}) \ge \bar l > 0
\end{equation}

The proof can be decomposed in the following three lemma. 
\begin{lemma}[Top-layer contraction]
\label{lemma:top-layer-contraction}
If ~\eqref{eq:weight-separation} holds for $t$, then ~\eqref{eq:v-contract}) holds for iteration $t + 1$.
\end{lemma}

\begin{lemma}[Bottom-layer contraction]
\label{lemma:bottom-layer-contraction}
If ~\eqref{eq:v-contract} holds for $t$, then ~\eqref{eq:wu-contract} holds for $t+1$ and ~\eqref{eq:wr-bound} holds for $t + 1$.
\end{lemma}

\begin{lemma}[Weight separation]
\label{lemma:weight-separation}
If ~\eqref{eq:weight-separation} holds for $t$, ~\eqref{eq:wr-bound} holds for $t+1$ and ~\eqref{eq:wu-contract} holds for $t+1$, 
then ~\eqref{eq:weight-separation} holds for $t + 1$. 
\end{lemma}

\begin{figure}
    \centering
    \includegraphics[width=\textwidth]{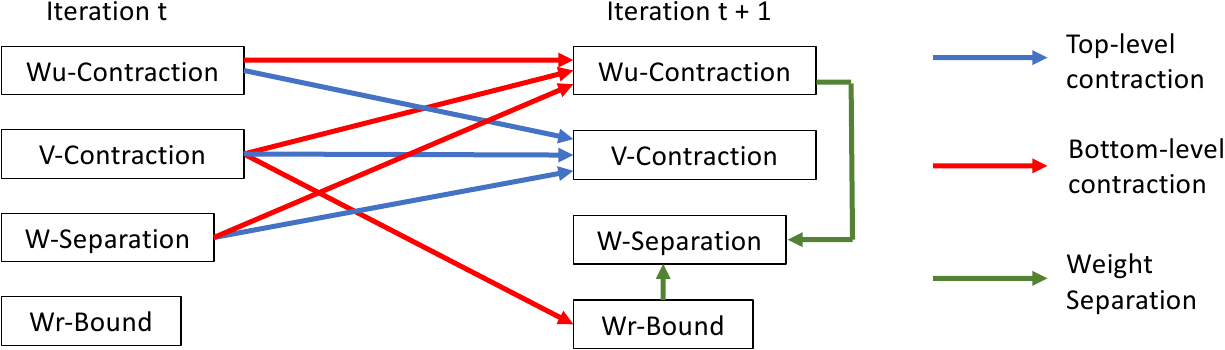}
    \caption{Proof sketch of Thm.~\ref{thm:over-parameterization}.}
    \label{fig:proof-sketch-overparameterization}
\end{figure}

As suggested by Fig.~\ref{fig:proof-sketch-overparameterization}, if all the three lemmas are true then the induction hypothesis are true.
\end{proof}

In the following, we will prove the three lemmas. 
\subsubsection{Lemma~\ref{lemma:top-layer-contraction}}
\begin{proof}
On the top-layer, we have $\dot V = L^*V^* - LV$. Denote that $V = \left[
    \begin{array}{c}
    \vv_1 \\
    \ldots \\
    \vv_n
    \end{array}
    \right]$, where $\vv_j$ is the $j$-th row of the matrix $V$. For each component, we can write:
\begin{equation}
    \dot\vv_j = \mathbb{I}(j \in [u])\vq^*_{jj} - \vq_{jj}
    + \sum_{j'\neq j, j'\in [u]} \Delta\vq_{jj'} 
    + \sum_{j'\neq j, j'\in [r]} \vq_{jj'}
\end{equation}
Note that there is no projection (if there is any, the projection should be in the columns rather than the rows). 

If~\eqref{eq:weight-separation} is true, we know that for $j\neq j'$,
\begin{eqnarray}
    \norm{\Delta\vq_{jj'}} \le \epsilon_l M_{l,uu}l^*_{jj}\norm{\delta\vv_{j'}}, \quad \norm{\vq_{jj'}} \le \epsilon_l M_{l,ur}l^*_{jj}\norm{\vv_{j'}}, \quad j \in [u] \\
    \norm{\Delta\vq_{jj'}} \le \epsilon_l M_{l,ru}l^*_{jj}\norm{\delta\vv_{j'}}, \quad \norm{\vq_{jj'}} \le \epsilon_l M_{l,rr}l^*_{jj}\norm{\vv_{j'}}, \quad j \in [r]
\end{eqnarray}
Now we discuss $j\in[u]$ and $j\in[r]$: 

\textbf{Relevant nodes}. For $j\in [u]$, the first two terms are: 
\begin{equation}
    \Delta\vq_{jj} = -l^*_{jj}\delta\vv_j + (l^*_{jj} - l_{jj})\vv_j
\end{equation}
From Lemma~\ref{lemma:quadratic-l} we know that:
\begin{equation}
    \norm{(l^*_{jj} - l_{jj})\vv_j} \le Cl^*_{jj}\norm{\delta\vw_j}^2\norm{\vv_j} \le 2C\sin(\theta_0/2)(1+B_{\delta v})l^*_{jj}\norm{\delta\vw_j} = \kappa l^*_{jj}\norm{\delta\vw_j} 
\end{equation}
Therefore using~\eqref{eq:v-contract} and ~\eqref{eq:wu-contract} at iteration $t$, we have: 
\begin{eqnarray}
    \norm{\delta\vv^{t+1}_j} &\le& (1 - \eta l^*_{jj})\norm{\delta\vv^{t}_j} + \eta \kappa l^*_{jj}\norm{\delta\vw^t_j} + \eta \epsilon_l M_{l,uu} l^*_{jj} \sum_{j'\neq j, j'\in [u]} \norm{\delta\vv^{t}_{j'}} + \eta \epsilon_l M_{l,ur}l^*_{jj}\sum_{j'\neq j, j'\in [r]}\norm{\vv^{t}_{j'}} \nonumber \\
    &\le& (1 - \eta \bar l \gamma)^{t+1} B_{\delta v}
\end{eqnarray}
Since $\gamma$ satisfies Eqn.~\ref{eq:v-cond}.

\textbf{Irrelevant nodes}. Note that for $j\in [r]$, we don't have the term $\vq^*_{jj}$. Therefore, we have:
\begin{eqnarray}
    \norm{\vv^{t+1}_j} &\le& (1 - \eta l_{jj})\norm{\vv^{t}_j} + \eta \epsilon_l M_{l,ru} l^*_{jj} \sum_{j'\neq j, j'\in [u]} \norm{\delta\vv^{t}_{j'}} + \eta \epsilon_l M_{l,rr} l^*_{jj} \sum_{j'\neq j, j'\in [r]}\norm{\vv^{t}_{j'}} \nonumber \\
    &\le& (1 - \eta l^*_{jj})\norm{\vv^{t}_j} + \eta\kappa l^*_{jj}\norm{\vv^{t}_j} + \eta \epsilon_l M_{l,ru} l^*_{jj} \sum_{j'\neq j, j'\in [u]} \norm{\delta\vv^{t}_{j'}} + \eta \epsilon_l M_{l,rr} l^*_{jj} \sum_{j'\neq j, j'\in [r]}\norm{\vv^{t}_{j'}} \nonumber \\
    &\le& (1 - \eta\bar l \gamma)^{t+1} B_v
\end{eqnarray}
\end{proof}

\subsubsection{Lemma~\ref{lemma:bottom-layer-contraction}}
\begin{proof}

Similar to the proof of Thm.~\ref{thm:close-convergence}, for node $j$, in the lower-layer, we have: 
\begin{equation}
    \dot\vw_j = \mathbb{I}(j \in [u])\proj{j} \tvp^*_{jj} + \proj{j}\sum_{j'\neq j, j'\in [u]}\Delta\tvp_{jj'} + \proj{j} \sum_{j'\in [r], j'\neq j} \tvp_{jj'}
\end{equation}
where $h_{jj'} = d_{jj'}\vv_j^T\vv_{j'}$ and $\tvp_{jj'} = \vp_{jj'} \vv_j^T\vv_{j'} = \vw_{j'}h_{jj'}$.

Due to~\eqref{eq:weight-separation} and $\norm{\vw_{j'}} = 1$, we know that for $j\neq j'$:
\begin{eqnarray}
    \norm{\Delta\tvp_{jj'}} \le \epsilon_d M_{d,uu} d^*_{jj}\norm{\delta\vw_{j'}}\norm{\vv_j}\norm{\vv_{j'}}, \ \ \norm{\tvp_{jj'}} \le \epsilon_d M_{d,ur} d^*_{jj}\norm{\delta\vw_{j'}}\norm{\vv_j}\norm{\vv_{j'}}, \quad j\in [u] \\
    \norm{\Delta\tvp_{jj'}} \le \epsilon_d M_{d,ru} d^*_{jj}\norm{\delta\vw_{j'}}\norm{\vv_j}\norm{\vv_{j'}}, \ \ \norm{\tvp_{jj'}} \le \epsilon_d M_{d,rr} d^*_{jj}\norm{\delta\vw_{j'}}\norm{\vv_j}\norm{\vv_{j'}}, \quad j\in [r]
    \label{eq:additional-term-w}
\end{eqnarray}
\yuandong{Some sin issues here} Note that if $\norm{\vv_{j'}}$(for $j\in [r]$) doesn't converge to zero, then due to Eqn.~\ref{eq:additional-term-w}, there is always residue and $\vw_j$ won't converge to $\vw^*_j$. 

Now we discuss two cases: 

\textbf{Relevant nodes}. For $j\in [u]$, similar to Eqn.~\ref{eq:cos-convergence} we have:
\begin{eqnarray}
    \sin\theta^{t+1}_j &\le& (1 - \eta d^*_{jj}\norm{\vv^t_j}^2 \cos \theta^t_j)\sin\theta^t_j + \eta \norm{\vv^t_j} \epsilon_d M_{d, uu} d^*_{jj}\sum_{j'\neq j, j\in[u]}\norm{\vv^t_{j'}} \sin\theta^t_{j'} \nonumber \\
    &+& \eta \norm{\vv^t_j} \epsilon_d M_{d, ur} d^*_{jj}\sum_{j'\neq j, j\in[r]} \norm{\vv^t_{j'}} 
\end{eqnarray}

Since~\eqref{eq:wu-contract} and ~\eqref{eq:v-contract} holds for time $t$, we know that:
\begin{equation}
    \sin\theta^{t+1}_j \le (1 - \eta\bar d \gamma)^{t+1}\sin\theta_0 
\end{equation}
since Eqn.~\ref{eq:w-cond} holds. 

\textbf{Irrelevant nodes}. In this case, we cannot prove for $j\in [r]$, $\vw_j$ converges to any determined target. Instead, we show that $\vw_j$ won't move too much from its initial location $\vw^0_j$, which is also set to be $\vw^*_j$, before its corresponding $\vv_j$ converges to zero. This is important to ensure that~\eqref{eq:weight-separation} remains correct thorough-out the iterations. 

For any $j\in [u]$, using~\eqref{eq:wu-contract} and~\eqref{eq:v-contract}, we know that the distance between the current $\vw_j$ and its initial value is 
\begin{eqnarray}
    \norm{\vw^{t+1}_j - \vw^0_j} &\le& \eta \sum_{t'=0}^t \norm{\dot \vw^{t'}_j} \le \eta \sum_{t'=0}^t\left\| \sum_{j'\neq j, j'\in [u]}\Delta\tvp^{t'}_{jj'} + \sum_{j'\in [r], j'\neq j} \tvp^{t'}_{jj'}\right\| \\
    &\le& \eta \epsilon_d B_{d,u} (B_v + B_{\delta v})B_v \sum_{t'=0}^t (1 - \eta\bar\lambda\gamma)^{2t'} \\
    &=& \frac{\epsilon_d B_{d,r} (B_v + B_{\delta v})B_v}{\bar\lambda\gamma(2 - \eta\bar\lambda\gamma)} = C_{d,r}
\end{eqnarray}
Therefore, we prove that~\eqref{eq:wr-bound} holds for iteration $t+1$. 
\end{proof}

\subsection{Lemma~\ref{lemma:weight-separation}}
\begin{proof}
Simply followed from combining Lemma~\ref{lemma:top-bounds}, Lemma~\ref{lemma:bottom-bounds} and weight bounds ~\eqref{eq:wu-contract} and~\eqref{eq:v-contract}.
\end{proof}

\begin{figure}[t]
    \centering
    \begin{subfigure}[t]{0.95\textwidth}
    \includegraphics[width=0.95\textwidth]{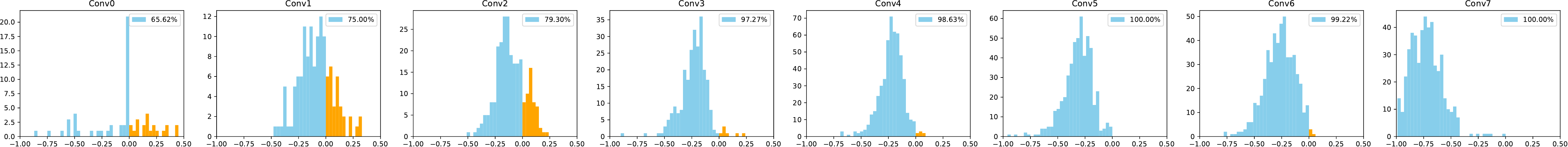}
    \caption{VGG11}
    \end{subfigure}
    
    \begin{subfigure}[t]{0.95\textwidth}
    \includegraphics[width=0.95\textwidth]{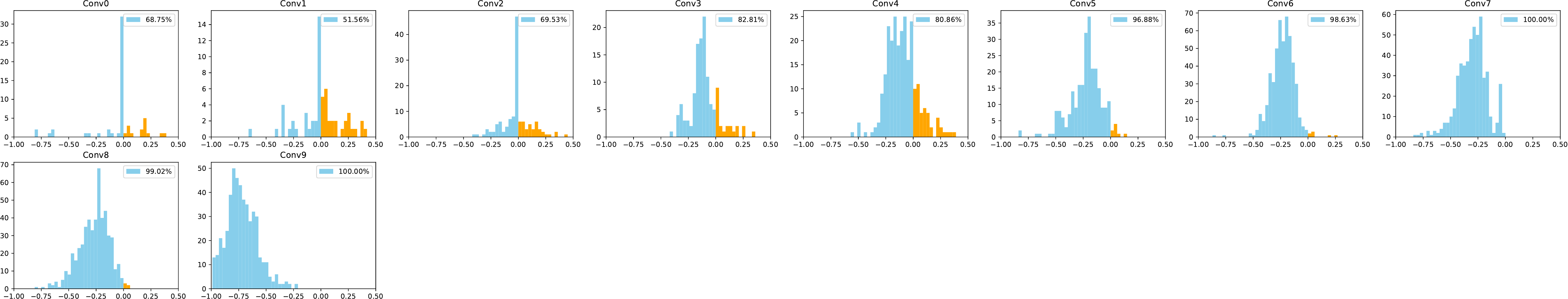}
    \caption{VGG13}
    \end{subfigure}

    \begin{subfigure}[t]{0.95\textwidth}
    \includegraphics[width=0.95\textwidth]{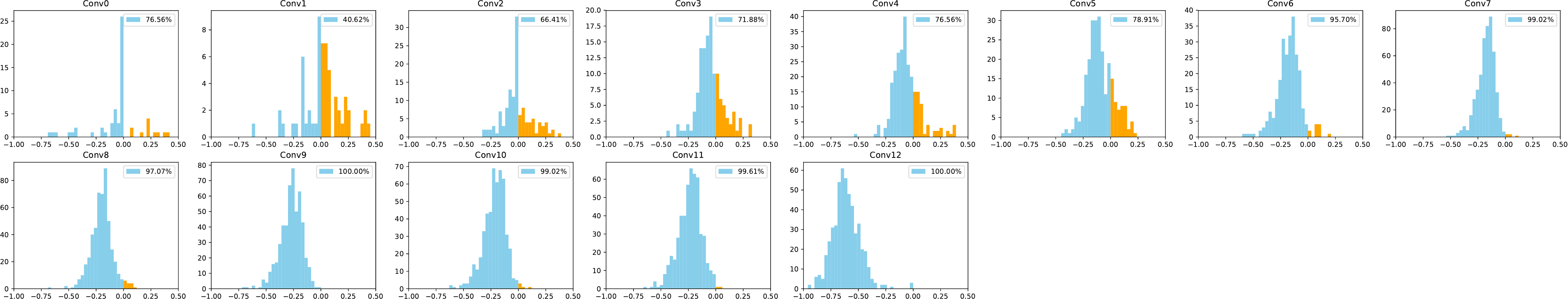}
    \caption{VGG16}
    \end{subfigure}
    
    \begin{subfigure}[t]{0.95\textwidth}
    \includegraphics[width=0.95\textwidth]{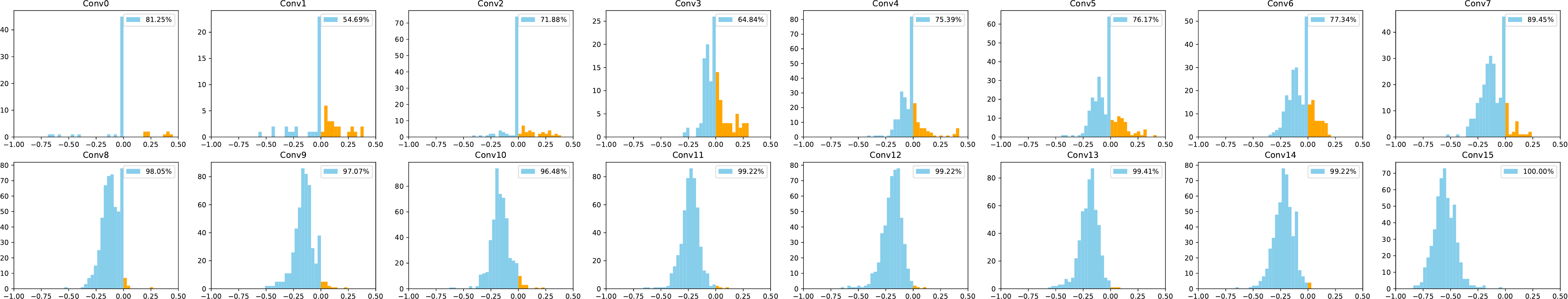}
    \caption{VGG19}
    \end{subfigure}

    \caption{\small{BatchNorm bias distribution of pre-trained VGG11/13/16/19 on ImageNet. Orange/blue are positive/negative biases. The first plot corresponds to the lowest layer (closest to the input).}}
    \label{fig:vgg11-appendix}
\end{figure}

\section{More experiments}
\begin{figure}[ht]
    \centering
    \includegraphics[width=0.48\textwidth]{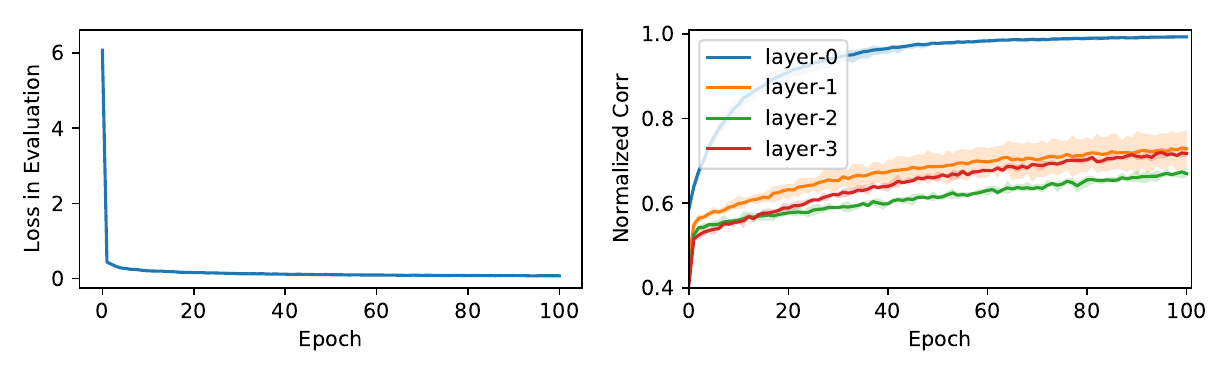}
    \includegraphics[width=0.48\textwidth]{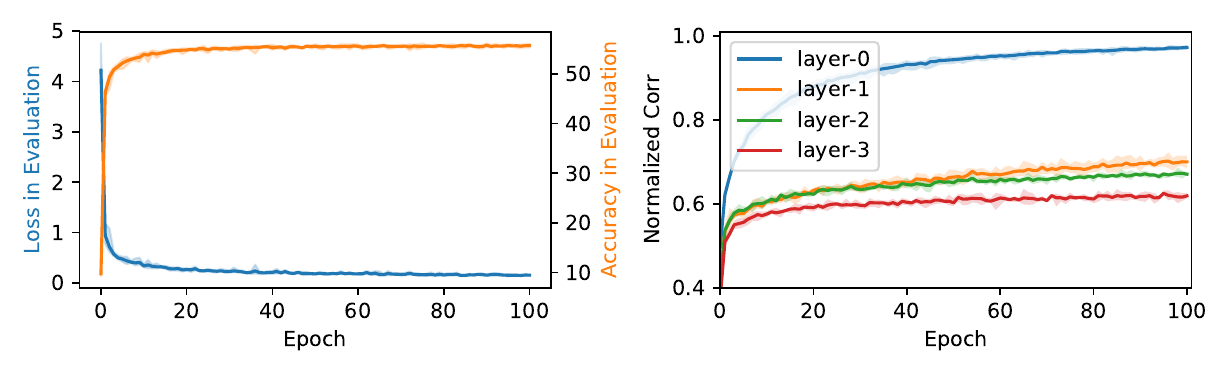}
    \includegraphics[width=0.48\textwidth]{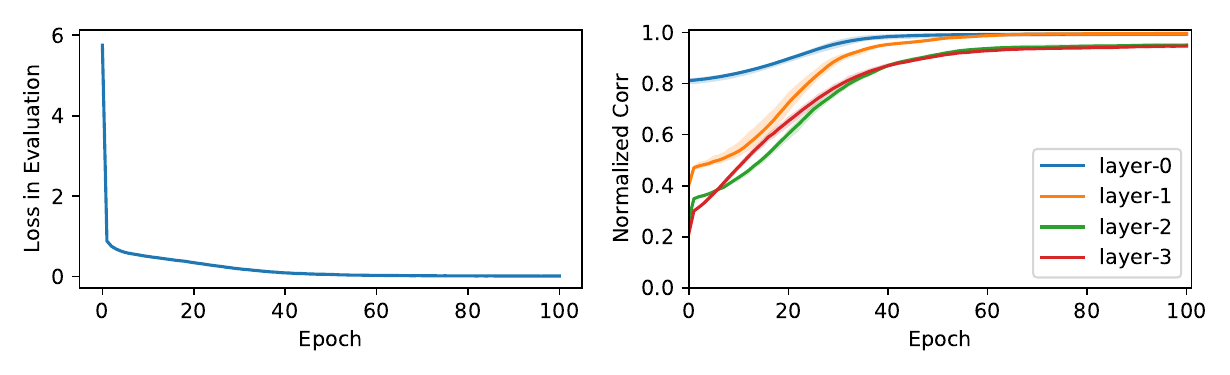}
    \includegraphics[width=0.48\textwidth]{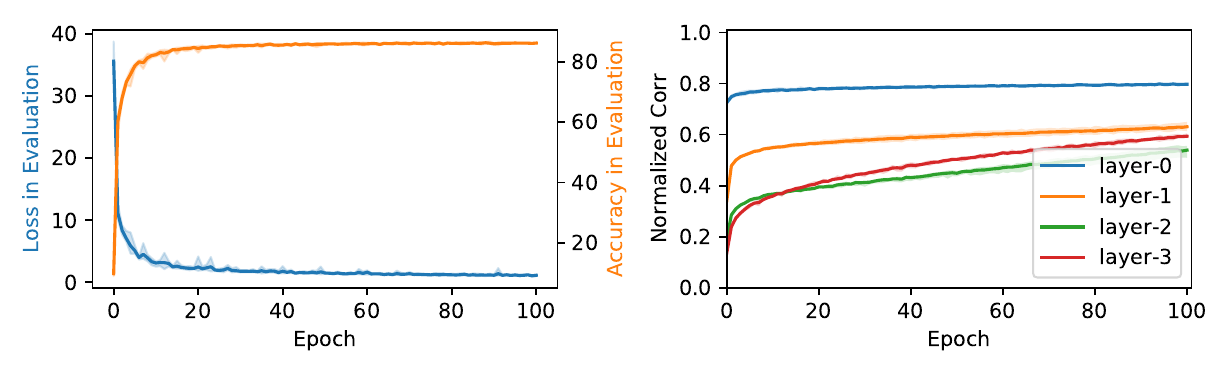}
    \vspace{-0.15in}
    \caption{\small{Loss and correlation between teacher and student nodes over optimization, all using BatchNorm. Gaussian (left) versus CIFAR10 (right). FC (top) versus CNN (bottom). \texttt{Layer-0} is the lowest layer (closest to the input). The mean best correlation steadily goes up over time.}}
    \label{fig:activation-corr-with-bn}
\end{figure}

\begin{figure}[ht]
    \centering
    \includegraphics[width=0.48\textwidth]{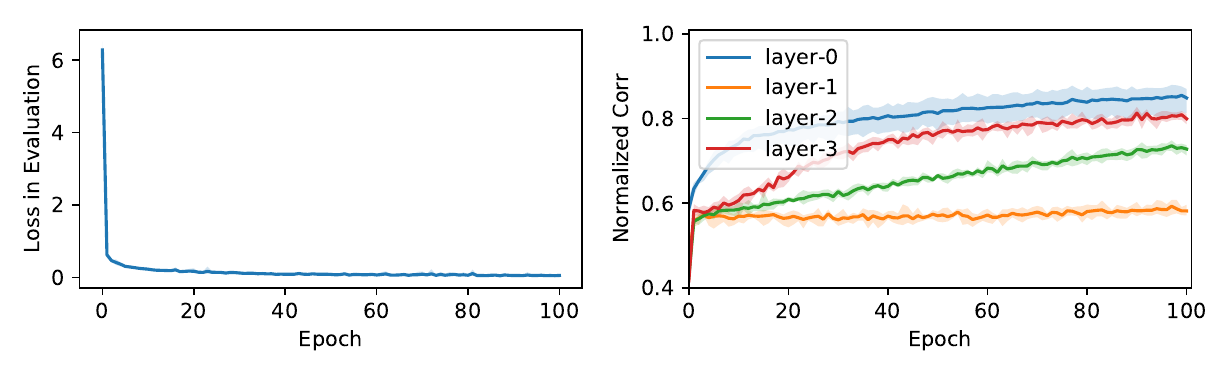}
    \includegraphics[width=0.48\textwidth]{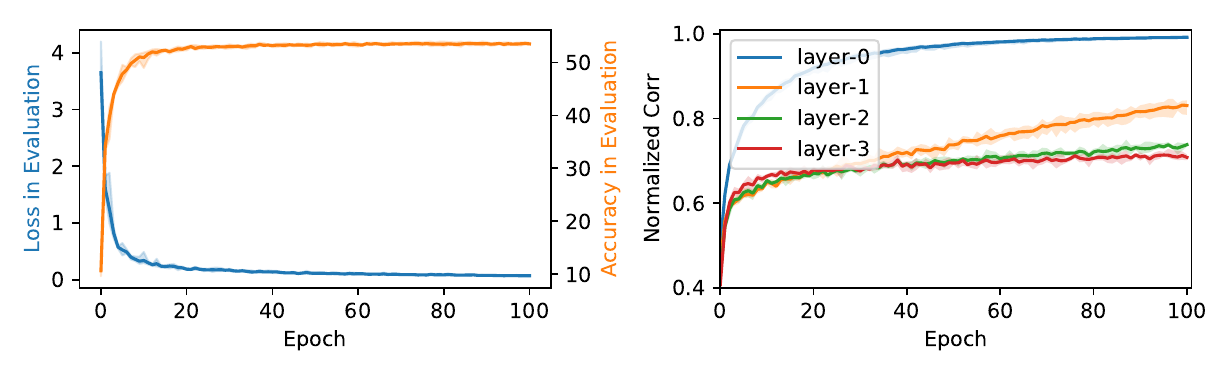}
    \includegraphics[width=0.48\textwidth]{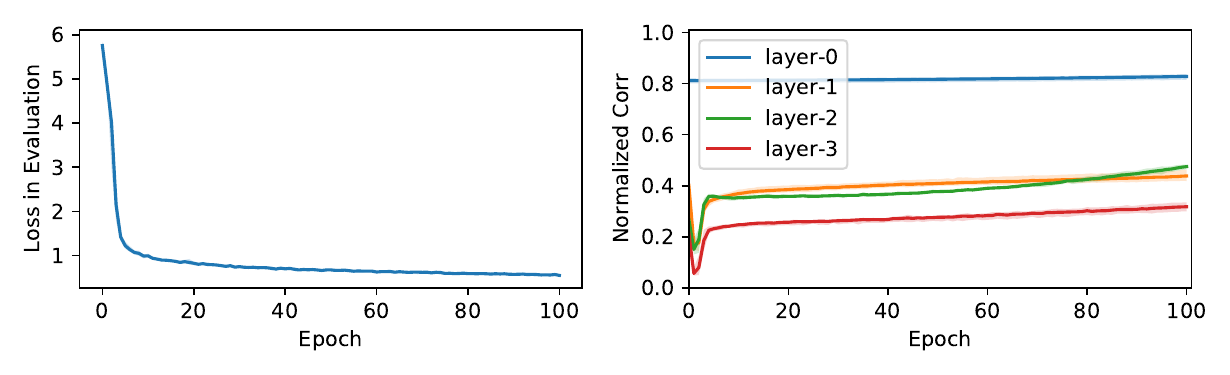}
    \includegraphics[width=0.48\textwidth]{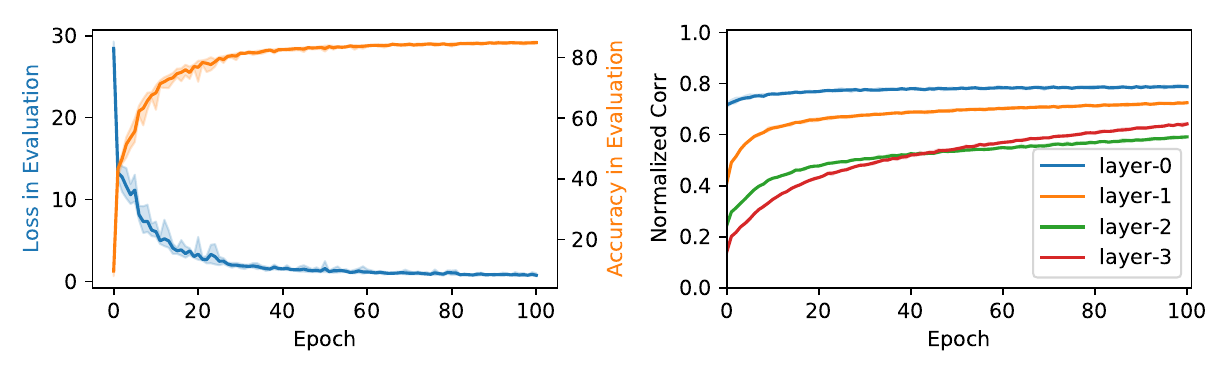}
    \vspace{-0.15in}
    \caption{\small{Same plots as Fig.~\ref{fig:activation-corr-with-bn} but trained \textbf{\emph{without}} BatchNorm.}}
    \label{fig:activation-corr-without-bn}
\end{figure}

\begin{figure}[ht]
    \centering
    \includegraphics[width=0.48\textwidth]{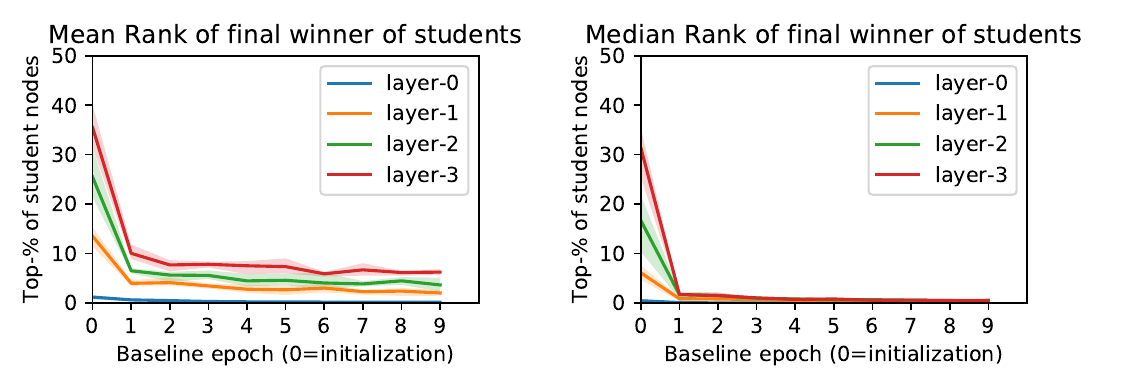}
    \includegraphics[width=0.48\textwidth]{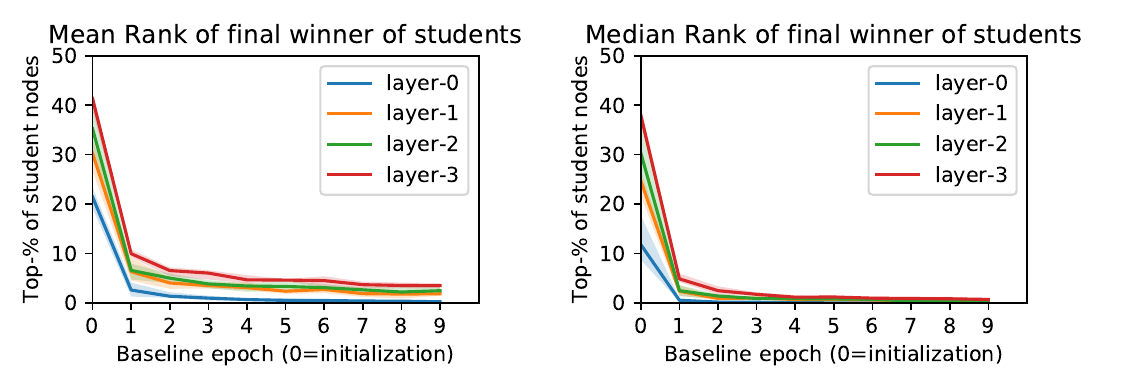}
    \includegraphics[width=0.48\textwidth]{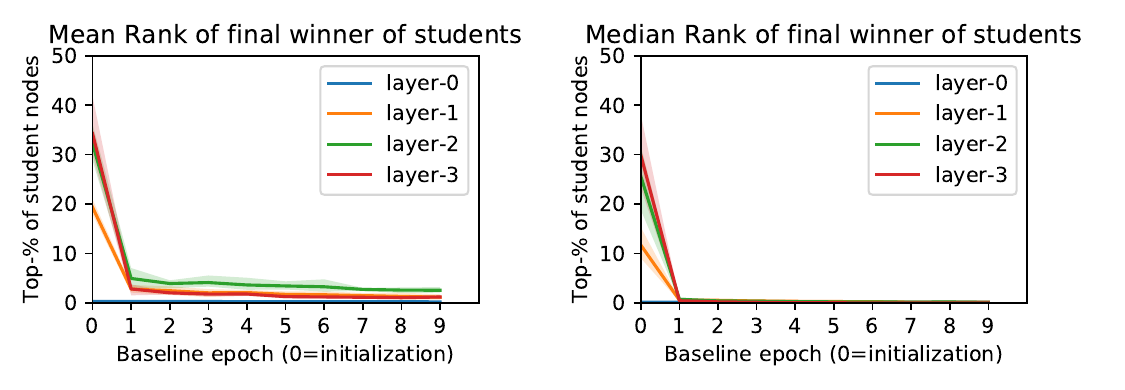}
    \includegraphics[width=0.48\textwidth]{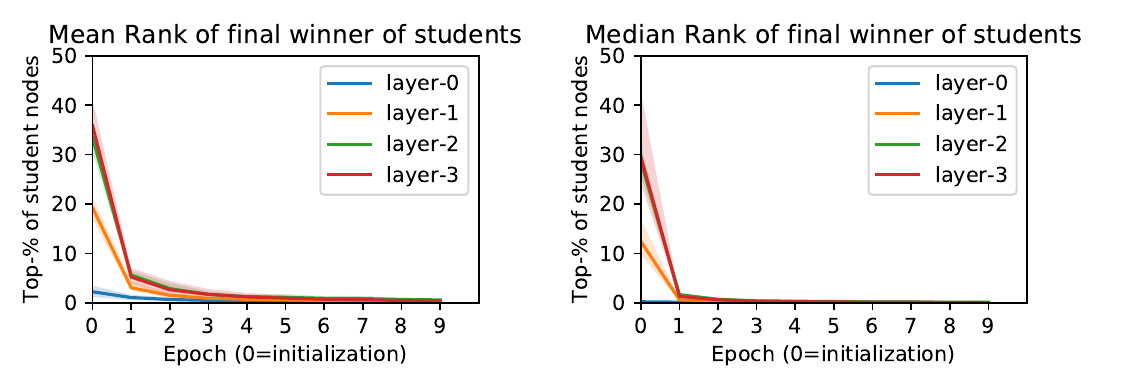}
    \caption{\small{Mean/Median rank at different epoch of the final winning student nodes that best match the teacher nodes after the training \emph{with} BatchNorm. Gaussian (left) and CIFAR10 (right). FC (top) and CNN (bottom). 
    \ifwithappendix
For training without BatchNorm, see Fig.~\ref{fig:mean-median-rank-without-bn}.
\fi
    }}
    \label{fig:mean-median-rank-with-bn}
\end{figure}

\begin{figure}
    \centering
    \includegraphics[width=0.48\textwidth]{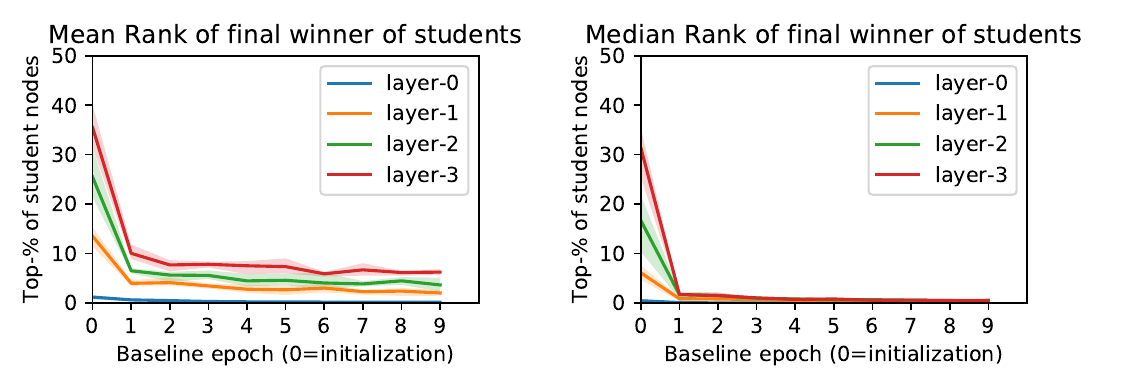}
    \includegraphics[width=0.48\textwidth]{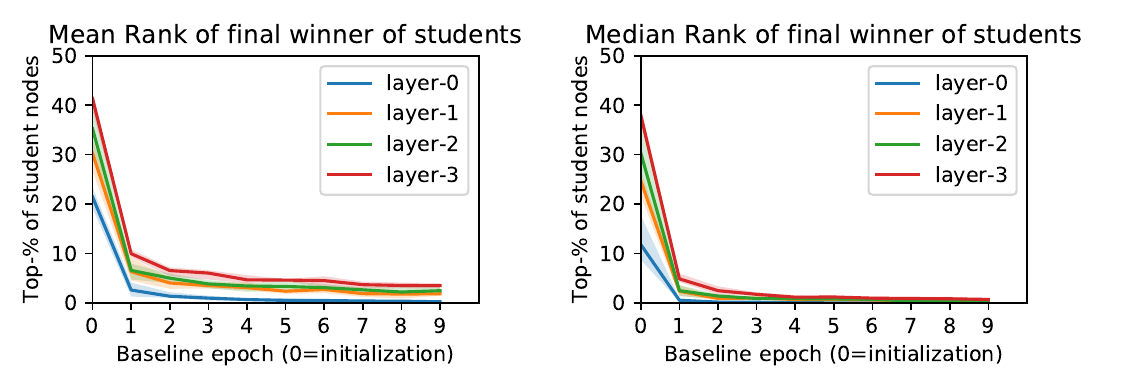}
    \includegraphics[width=0.48\textwidth]{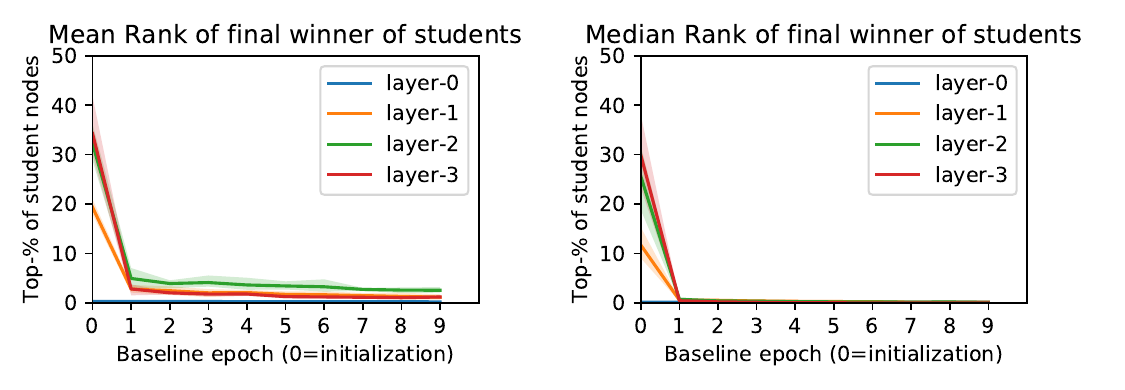}
    \includegraphics[width=0.48\textwidth]{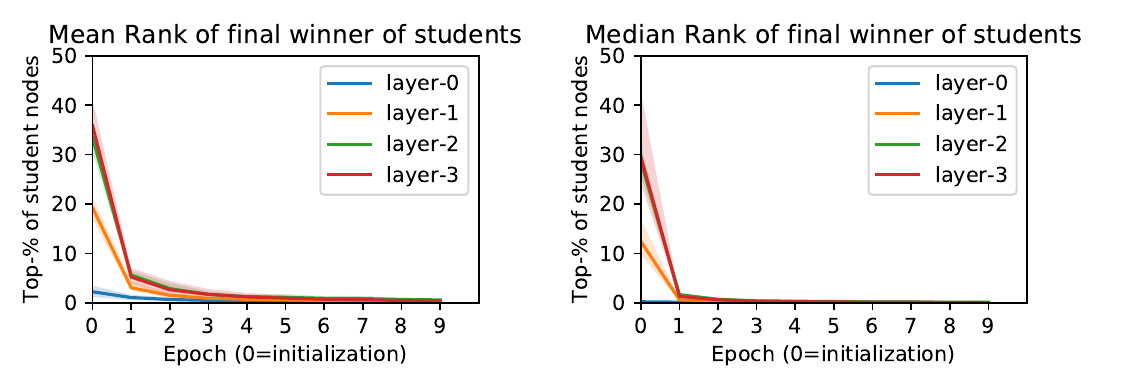}
    \caption{\small{Mean/Median rank at different epoch of the final winning student nodes that best match the teacher nodes after the training using BatchNorm. Gaussian (left) versus CIFAR10 (right). FC (top) versus CNN (bottom).}}
    \label{fig:mean-median-rank-with-bn-appendix}
\end{figure}

\begin{figure}
    \centering
    \includegraphics[width=0.48\textwidth]{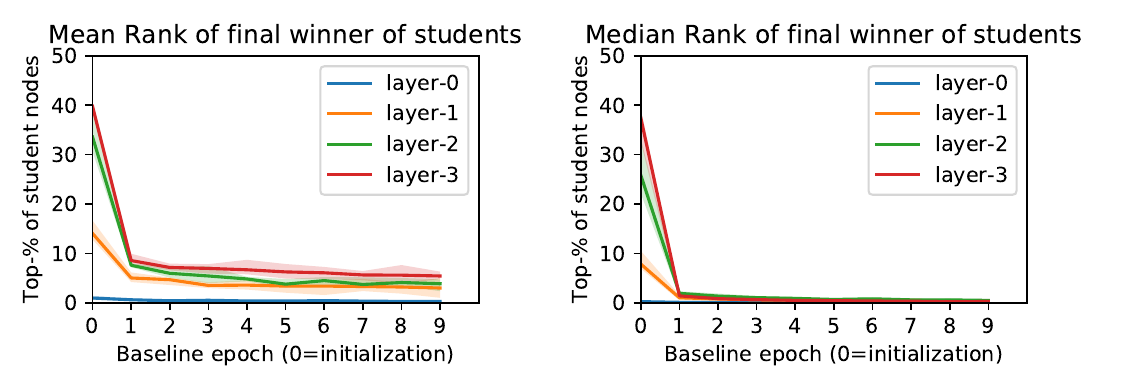}
    \includegraphics[width=0.48\textwidth]{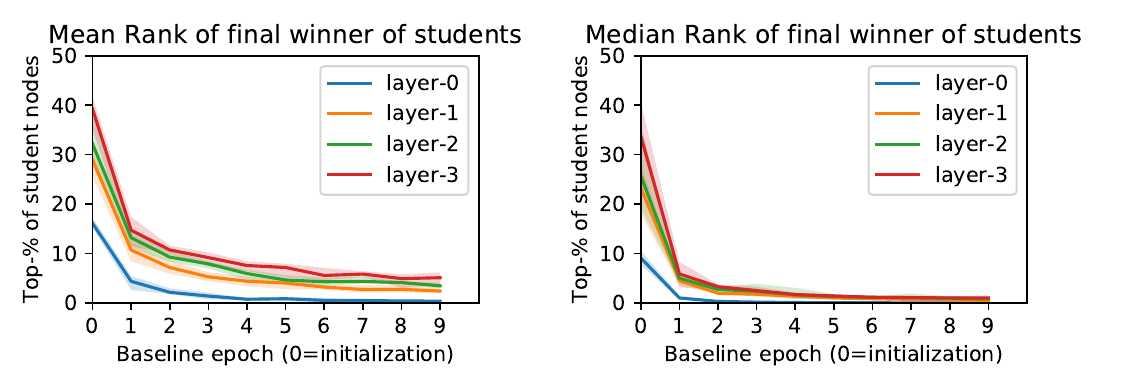}
    \includegraphics[width=0.48\textwidth]{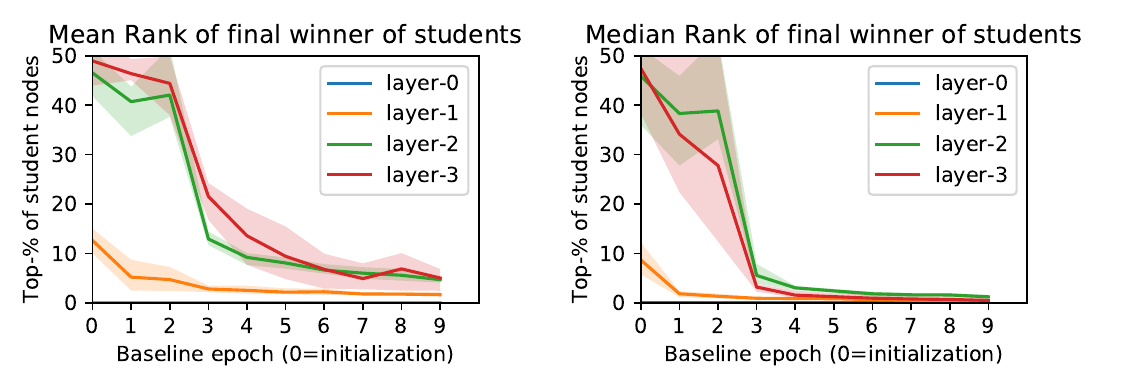}
    \includegraphics[width=0.48\textwidth]{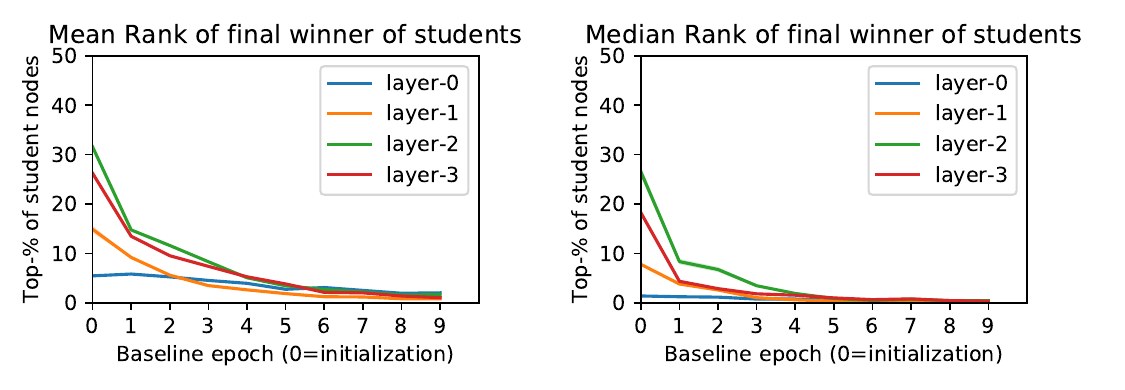}
    \caption{\small{Mean/Median rank at different epoch of the final winning student nodes that best match the teacher nodes after the training \emph{without} BatchNorm. Gaussian (left) versus CIFAR10 (right). FC (top) versus CNN (bottom).}}
    \label{fig:mean-median-rank-without-bn}
\end{figure}

\begin{figure}
    \centering
    \includegraphics[width=0.48\textwidth]{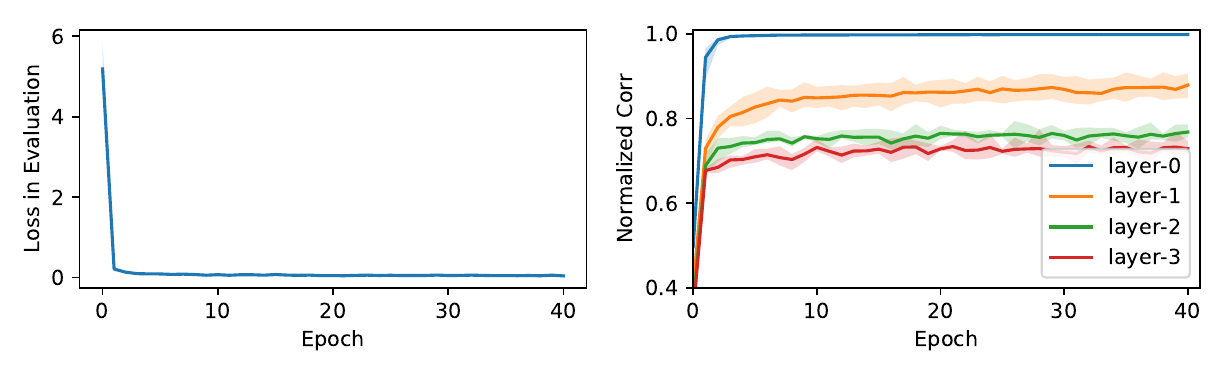}
    \includegraphics[width=0.48\textwidth]{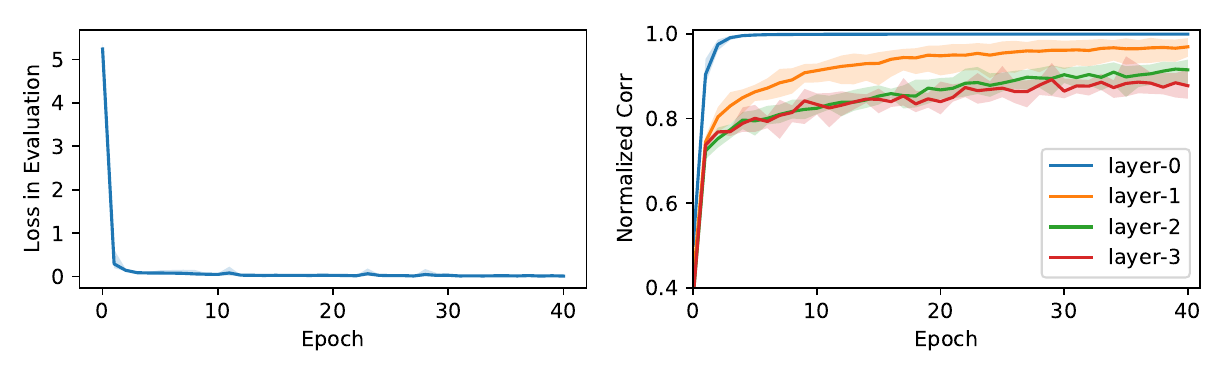}
    \includegraphics[width=0.48\textwidth]{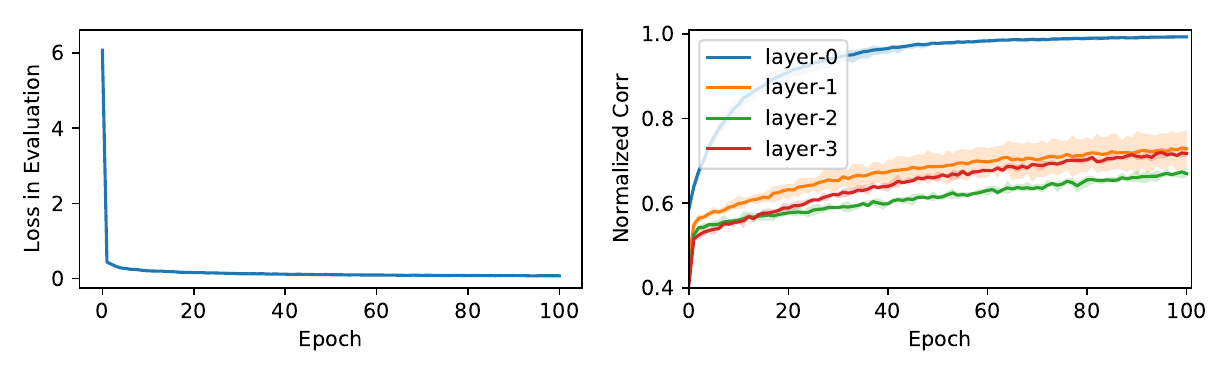}
    \includegraphics[width=0.48\textwidth]{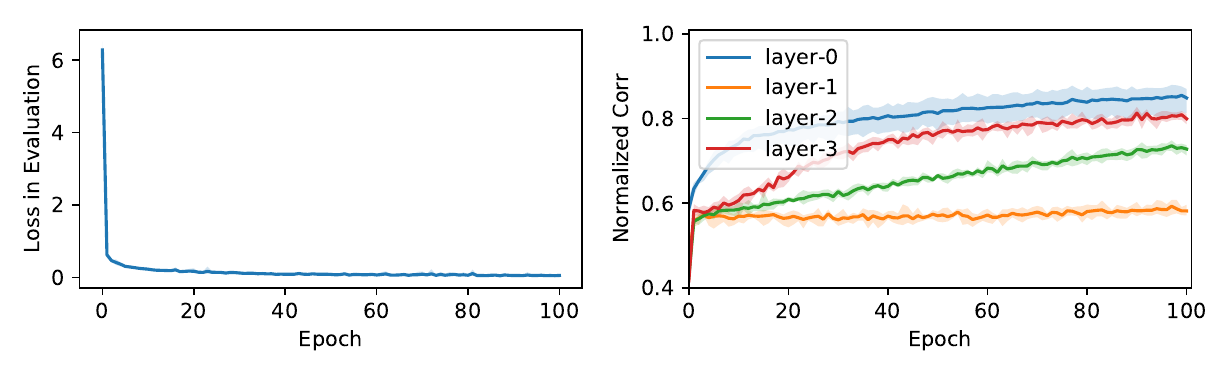}
    \caption{Gaussian data with small (10-15-20-25) and large (50-75-100-125) FC models. Small models (top) versus large models (bottom). With BN (left) versus Without BN (right).}
    \label{fig:small-large-model-appendix}
\end{figure}

\begin{figure}
    \centering
    \includegraphics[width=0.48\textwidth]{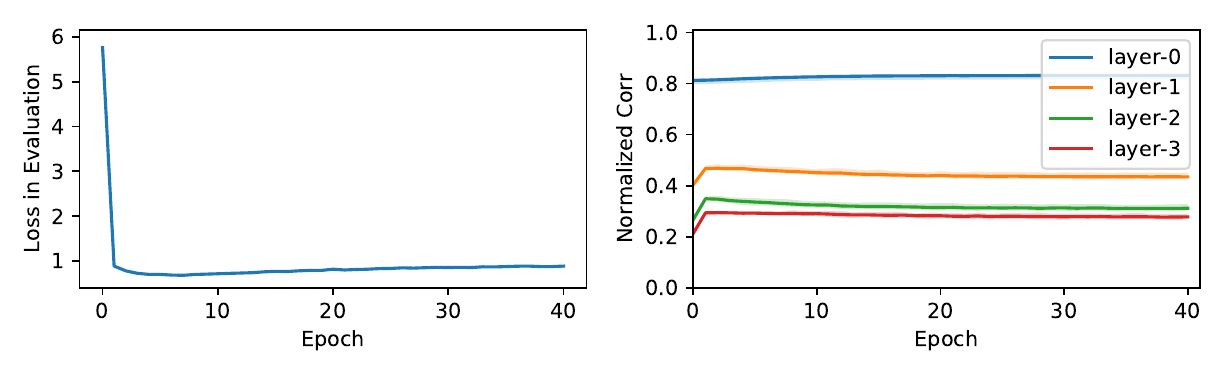}
    \includegraphics[width=0.48\textwidth]{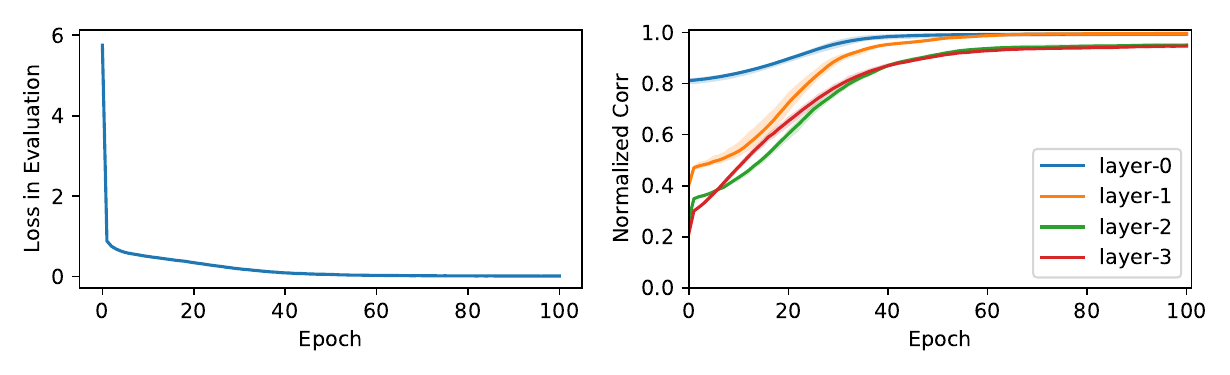}
    \includegraphics[width=0.48\textwidth]{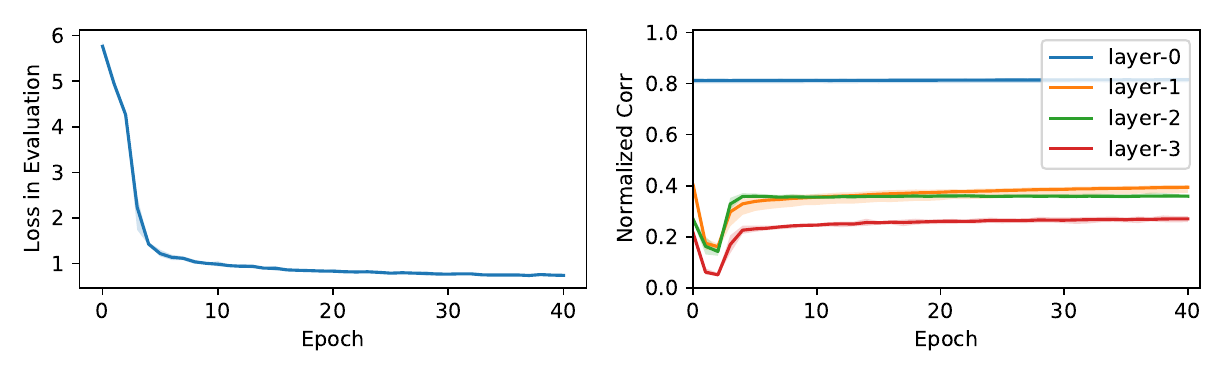}
    \includegraphics[width=0.48\textwidth]{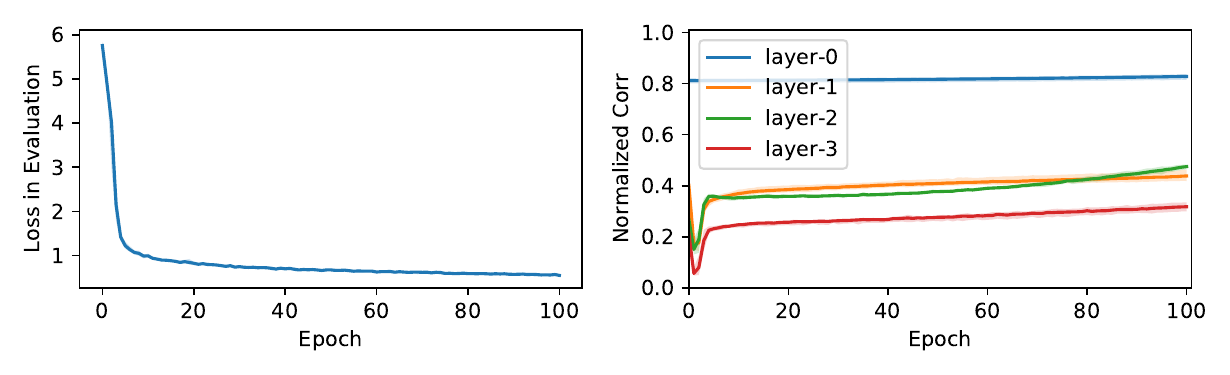}
    \caption{Gaussian CNN. With BN (top) versus Without BN (bottom). Finite Dataset (left) versus Infinite Dataset (right).}
    \label{fig:finite-dataset-appendix}
\end{figure}

\end{document}